\documentclass[nohyperref]{article}

\usepackage{microtype}
\usepackage{graphicx}
\usepackage{booktabs} % for professional tables
\usepackage[export]{adjustbox}
\usepackage{hyperref}
\usepackage{placeins}
\PassOptionsToPackage{font=small}{caption}

\PassOptionsToPackage{dvipsnames}{xcolor}
\usepackage[accepted]{icml2023}
\usepackage{amsmath}
\usepackage{amssymb}
\usepackage{mathtools}
\usepackage{amsthm}
\usepackage[shortlabels]{enumitem}
\usepackage{subcaption}
\usepackage{afterpage}
\usepackage{makecell}

\usepackage[capitalize,noabbrev]{cleveref}
\PassOptionsToPackage{numbers}{natbib}
\usepackage{macros}
\newcommand {\maze}{\texorpdfstring{\texttt{maze2d}\@\xspace}{maze2d}}
\newcommand {\mntcar}{\texorpdfstring{\texttt{MountainCar}\@\xspace}{MountainCar}}
\newcommand {\dFOURrl}{\texorpdfstring{\texttt{d4rl}\@\xspace}{d4rl}}

\newcommand {\mqmet}{\mathfrak{Qmet}}

\definecolor{MyColorPres}{RGB}{26,82,255}
\colorlet{MyColorMaxi}{teal}
\definecolor{myorange}{RGB}{230,97,0}
\colorlet{MyColorQmet}{myorange}
\newcommand*{\presC}[1]{\textcolor{MyColorPres}{\emph{\ul{#1}}}\xspace}
\newcommand*{\maxiC}[1]{\textcolor{MyColorMaxi}{\emph{\ul{#1}}}\xspace}
\newcommand*{\qmetC}[1]{\textcolor{MyColorQmet}{\emph{\ul{#1}}}\xspace}
\newcommand*{\optC}[1]{\textcolor{red}{\emph{\ul{#1}}}\xspace}
\renewcommand*{\optC}[1]{\emph{\ul{#1}}\xspace}

\newcommand{\hide}[1]{}

\newcommand{\exul}[1]{\expandafter\ul\expandafter{#1}}

\newcommand*{\eqnameformat}[1]{%
  \textsf{#1}%
}

\makeatletter
\@ifdefinable{\org@maketag@@@}{%
  \let\org@maketag@@@\maketag@@@
  \renewcommand*{\maketag@@@}[1]{%
    \org@maketag@@@{%
      \@ifundefined{eq@name}{#1}{%
        \begin{tabular}[t]{@{}r@{}}%
          #1\tabularnewline
          \eqnameformat{\@nameuse{eq@name}}%
        \end{tabular}%
      }%
    }%
  }%
}
\newif\ifeqname@star
\newcommand*{\eqname}{%
  \@ifstar{\eqname@startrue\eqname@}{\eqname@starfalse\eqname@}%
}
\newcommand*{\eqname@}[2][]{%
  \gdef\eq@name{#1}%
  \ifx\eq@name\@empty
  \else
    \begingroup
      \@ifundefined{GetTitleString}{%
        \gdef\@currenteqlabelname{#2}%
      }{%
        \GetTitleString{#2}%
        \global\let\@currenteqlabelname\GetTitleStringResult
      }%
      \let\@currentlabelname\@currenteqlabelname
      \label{#1}%
    \endgroup
  \fi
  \gdef\eq@name{#2}%
  \ifx\eq@name\@empty
    \global\let\eq@name\relax
  \else
    \ifeqname@star
      \gdef\eq@name{\llap{#2}}%
    \fi
  \fi
}
\@ifdefinable{\org@make@display@tag}{%
  \let\org@make@display@tag\make@display@tag
  \def\make@display@tag{%
    \@ifundefined{@currenteqlabelname}{}{%
      \let\@currentlabelname\@currenteqlabelname
    }%
    \org@make@display@tag
  }%
}
\let\eq@name\relax
\let\@currenteqlabelname\relax
\g@addto@macro\displ@y@{%
  \global\let\eq@name\relax
  \global\let\@currenteqlabelname\relax
}
\@ifdefinable{\org@math@cr@@}{%
  \let\org@math@cr@@\math@cr@@
  \def\math@cr@@[#1]{%
    \org@math@cr@@[{#1}]%
    \noalign{%
      \global\let\eq@name\relax
    }%
  }%
}
\@ifdefinable{\org@eqref}{%
  \let\org@eqref\eqref
  \renewcommand*{\eqref}[1]{%
    \begingroup
      \let\eq@name\relax
      \org@eqref{#1}%
    \endgroup
  }%
}
\g@addto@macro\equation{%
  \eqname{}%
}
\makeatother

\usepackage{colortbl}
\usepackage{booktabs,arydshln}
\definecolor{notquitelightgray}{rgb}{0.67, 0.67, 0.67}

\makeatletter
\def\adl@drawiv#1#2#3{%
        \hskip.5\tabcolsep
        \xleaders#3{#2.5\@tempdimb #1{1}#2.5\@tempdimb}%
                #2\z@ plus1fil minus1fil\relax
        \hskip.5\tabcolsep}
\newcommand{\cdashlinelr}[1]{%
  \noalign{\vskip\aboverulesep
           \global\let\@dashdrawstore\adl@draw
           \global\let\adl@draw\adl@drawiv}
  \cdashline{#1}
  \noalign{\global\let\adl@draw\@dashdrawstore
           \vskip\belowrulesep}}
\makeatother

\setlist[itemize]{leftmargin=11pt,topsep=-3pt,itemsep=-3pt,parsep=3pt}
\setlist[itemize]{leftmargin=11pt,topsep=-2.5pt,itemsep=-1pt,parsep=3.5pt}

\usepackage{etoolbox}
\newcommand{\zerodisplayskips}{%
  \setlength{\abovedisplayskip}{5pt}%
  \setlength{\belowdisplayskip}{6pt}%
  \setlength{\abovedisplayshortskip}{5pt}%
  \setlength{\belowdisplayshortskip}{6pt}}
\appto{\normalsize}{\zerodisplayskips}
\appto{\small}{\zerodisplayskips}
\appto{\footnotesize}{\zerodisplayskips}

\icmltitlerunning{Quasimetric RL}

\renewcommand{\postsubmission}[2]{#2}

\begin{document}

\twocolumn[
% \icmltitle{Quasimetric Learning as Optimal Goal-Conditional Reinforcement Learning}
\icmltitle{Optimal Goal-Reaching Reinforcement Learning via Quasimetric Learning}



\icmlsetsymbol{equal}{*}

\begin{icmlauthorlist}
\icmlauthor{Tongzhou Wang}{mit}
\icmlauthor{Antonio Torralba}{mit}
\icmlauthor{Phillip Isola}{mit}
\icmlauthor{Amy Zhang}{utaustin,meta}
\end{icmlauthorlist}

% \icmlaffiliation{mit}{Massachusetts Institute of Technology}
% \icmlaffiliation{meta}{Meta AI}
% \icmlaffiliation{utaustin}{University of Texas at Austin}
\icmlaffiliation{mit}{MIT}
\icmlaffiliation{meta}{Meta AI}
\icmlaffiliation{utaustin}{UT Austin}


% \icmlaffiliation{mit}{MIT CSAIL}
% \icmlaffiliation{uw}{University of Washington}
% \icmlaffiliation{berkeley}{UC Berkeley}
% \icmlaffiliation{metaai}{Meta AI}

\icmlcorrespondingauthor{Tongzhou Wang}{tongzhou@mit.edu}

\icmlkeywords{Machine Learning, Reinforcement Learning, Quasimetrics, Geometry}

\vskip 0.3in
]



\printAffiliationsAndNotice{}  % leave blank if no need to mention equal contribution

\begin{abstract}
%!TEX root = ../main.tex

% \vspace*{-1.5pt}

% The ability to reason with high-level abstractions is critical to general intelligence. Humans learn and discover such abstractions as we interact with other people and the physical world. How can artificial agents do the same?  What kind of information can agents safely discard in such abstractions?  In this work, we categorize different types of information out in the wild, providing a framework that allows easy understanding of information removed by various prior work on representation learning in reinforcement learning (RL). We propose a novel approach for learning a world model that explicitly factors out certain noisy distractors. Extensive experiments on variants of DeepMind Control Suite and \robodesk demonstrate superior performance of our learned world model over both raw observation and prior works, across  policy optimization control tasks as well as the non-control task of joint position regression.

In goal-reaching reinforcement learning (RL), the optimal value function has a particular geometry, called \emph{quasimetric} structure. This paper introduces Quasimetric Reinforcement Learning (QRL), a new RL method that utilizes \qmet models to learn \optC{optimal} value functions. Distinct from prior approaches, the QRL objective is specifically designed for \qmets, and provides strong theoretical recovery guarantees. Empirically, we conduct thorough analyses on a discretized \mntcar environment, identifying properties of QRL and its advantages over alternatives. On offline and online goal-reaching benchmarks, QRL also demonstrates improved sample efficiency and performance, across both state-based and image-based observations. 

\vspace*{-0.5pt}
\begingroup%
\fontsize{8.5pt}{10pt}\selectfont%
\hspace{-21pt}
\noindent\begin{tabular}{@{}lr@{}}
\textbf{\fontsize{9.5pt}{10pt}\selectfont Project Page:}\hspace{0.2em} & \hspace{-1.025em}\href{https://tongzhouwang.info/quasimetric_rl/}{\texttt{tongzhouwang.info/quasimetric\_rl}}\\[0.65ex]
{\textbf{\fontsize{9.5pt}{10pt}\selectfont Code:}} & \hspace{-3.6em}%
{
\href{https://github.com/quasimetric-learning/quasimetric-rl/}%
{\texttt{github.com/quasimetric-learning/quasimetric-rl}}%
}
\\
\end{tabular}%
\endgroup%
\vspace*{-6pt}

% \begingroup%
% \fontsize{8.25pt}{10pt}\selectfont%
% \hspace{-17pt}
% \noindent\begin{tabular}{@{}lr@{}}
% \textbf{\fontsize{8.75pt}{10pt}\selectfont Project Page:}\hspace{0.2em} & \hspace{-1em}\href{https://ssnl.github.io/quasimetric_rl/}{\texttt{ssnl.github.io/quasimetric\_rl}} \\[0.425ex]
% {\textbf{\fontsize{8.75pt}{10pt}\selectfont Code:}} & \hspace{-3.5em}%
% {
% \href{https://github.com/quasimetric-learning/quasimetric_rl/}%
% {\texttt{github.com/quasimetric-learning/quasimetric\_rl}}%
% }
% \\
% \end{tabular}%
% \endgroup%
% \vspace*{-7.5pt}

% \vspace*{-5.5pt}
% \phil{say something about also outperforming in the single-goal setting?}

% ChatGPT's version:
% ``Modeling decision-making problems through dynamic programming often involves using the cost-to-go function, also known as the value function. This function allows for breaking down complex decisions into smaller subproblems and collapsing them into a single node with a summarized cost. However, in the multi-task setting, where the value function models the cost-to-go to a set of goals, additional structure emerges. The true optimal value function in this setting is always a quasimetric function. Despite this, previous RL algorithms have not exploited this fact, but our new algorithm aims to change that by constraining the value function search to the space of quasimetrics and redesigning the entire RL optimization process for more efficiency and stronger guarantees."
\end{abstract}

% \input{text/OUTLINE_A}
%!TEX root = ../main.tex
\section{Introduction}\label{sec:introduction}
% \vspace{-2.5pt}

Modern decision-making problems often involve dynamic programming on the cost-to-go function, also known as the value function. This function allows for bootstrapping, where a complicated decision is broken up into a series of subproblems. Once a subproblem is solved, its subgraph can be collapsed into a single node whose cost is summarized by the value function. This approach appears in nearly all contemporary RL and planning algorithms. 
% \vspace{-1pt}

In deep RL, value functions are modeled with general neural nets, which are universal function approximators. Further, most RL algorithms focus on optimizing toward a single goal. In this setting, the value function $V^*(s)$ reports the (optimal) cost-to-go to achieve that single goal from state $s \in \mathcal{S}$. It is known that $V^*$ can be any function $V^*\colon \mathcal{S} \rightarrow \R$, that is, for any $V^*\colon \mathcal{S} \rightarrow \R$, there exists a Markov Decision Process (MDP) for which that $V^*$ is the desired optimal value function. 
% This justifies using very flexible neural networks to parametrize $V^*$ in learning.

However, an additional structure emerges when we switch 
% instead 
to the multi-task setting, where the (goal-conditioned) value function $V^*(s;g) \colon \mathcal{S} \times \mathcal{S} \rightarrow \R$ is the cost-to-go to a given goal state $g$ (\Cref{fig:structure-multi-goal}). In this case, the optimal value function, for \emph{any} MDP, is always a quasimetric function~\cite{wang2022learning,sontag1995abstract,tian2020model,liu2022metric}, which is a generalization of metric functions to allow asymmetry while still respecting the triangle inequality. 

Given this structure, it is natural to constrain value function search to the space of quasimetrics. This approach searches within a much smaller subset of the space of all functions $\mathcal{S} \times \mathcal{S} \rightarrow \R$, ensuring that the true value function is guaranteed to be present within this subspace. Recent advancements in differentiable parametric quasimetric models \citep{wang2022iqe,pitis2020inductive} have already enabled a number of studies \citep{liu2022metric,wang2022learning} to explore the use of these models in standard RL algorithms, resulting in improved performance in some cases.

However, traditional RL algorithms (such as Q-learning \citep{watkins1989learning}) were designed for large unconstrained function spaces, and their performance may severely degrade with restricted spaces \citep{wang2020statistical,wang2021instabilities}. Instead of constraining the search space, they encourage quasimetric properties via the objective function. For example, the Bellman update partly enforces the triangle inequality on the current state, next state, and target goal. With the advent of differentiable parametric quasimetric models, these properties come for free with the architecture, and we aim to design a new RL algorithm specifically geared towards learning quasimetric value functions.
% and their demonstrated effectiveness in RL, it is newly possible to design a new RL algorithm that is specifically geared towards learning quasimetric value functions.

In this work, we propose \emph{Quasimetric Reinforcement Learning} (QRL). QRL is in the family of geometric approaches to value function learning, which model the value function as some distance metric, or, in our case, a quasimetric. Obtaining local distance estimates is easy because the cost of a single transition is by definition given by the reward function, and can be learned via regression towards observed rewards. 
However, capturing global relations is hard. This is a problem studied in many fields such as metric learning \citep{roweis2000nonlinear,tenenbaum2000global}, contrastive learning \citep{oord2018representation,wang2020hypersphere}, \etc. A general principle is to find a model where local relationships are captured and otherwise states are spread out. 

We argue that a similar idea can be used for value function learning. QRL finds a \qmetC{quasimetric} value function in which \presC{local distances are preserved}, but otherwise states are \maxiC{maximally spread out}.  All \emph{\ul{three properties}} are essential in accurately learning the function. Intuitively, a \qmet  function that maximizes the separation of states $s_0$ and $s_1$, subject to the constraint that it \presC{captures cost} for each adjacent pair of states, gives exactly the cost of the shortest path from $s_0$ to $s_1$. It can't be longer than that due to triangle inequality from \qmetC{quasimetric} and \presC{preservation of local distances}. It can't be shorter than that due to the \maxiC{maximal spreading}. Analogously, consider a chain with several links.  If one pushes the chain ends apart, then the distance between the ends is exactly equal to the length of all the links.

These \emph{\ul{three properties}}  (which we will indicate with the three text colors above) make our method distinct from other contrastive approaches to RL, and ensure that QRL \emph{provably} learns the \optC{optimal} value function.  Some alternatives use symmetrical metrics that cannot  capture complex dynamics \citep{yang2020plan2vec,ma2022vip,sermanet2018time}. Others do not enforce how much adjacent states are pulled together nor how much states are pushed apart, and rely on carefully weighting loss terms and specific sample distributions to estimate on-policy (rather than optimal) values \citep{eysenbach2022contrastive,oord2018representation}.

\begin{figure}
    \centering
    \vspace{-3.5pt}%
    \includegraphics[scale=0.218, trim=145 385 680 200, clip]{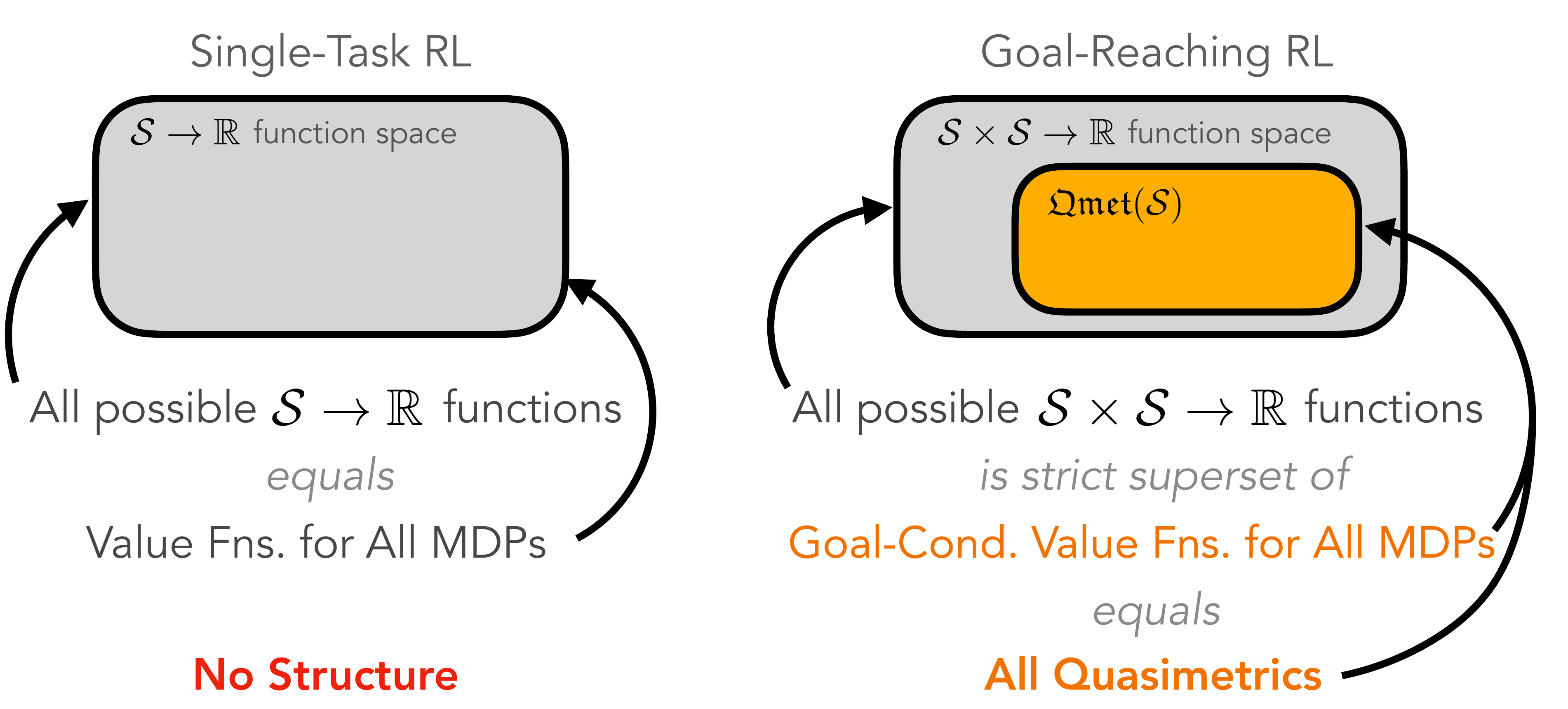}%
    \vspace{-8pt}%
    \caption{In multi-goal RL, the set of all possible (optimal) value functions is exactly the set of \qmetC{quasimetrics}. In single-task RL, there is no similar structure and value functions can be any function. }%
    \label{fig:structure-multi-goal}%
    \vspace{-6.25pt}
\end{figure}

% Quasimetrics is the correct way to do RL.

In summary, our contributions in this paper are  \begin{itemize}[topsep=-3.75pt]
    \item Based on the connection between value functions and quasimetrics (\Cref{sec:val-fn-qmet}), we propose QRL, a new RL framework that utilizes \emph{quasimetric} models to learn \emph{optimal} goal-reaching value functions (\Cref{sec:qrl}). 
    \item We provide theoretical guarantees (\Cref{sec:qrl-opt-intuition}) as well as thorough empirical analysis on a discretized \mntcar environment (\Cref{sec:disc-mnt-car}), highlighting qualitative differences with many existing methods. 
    \item We augment the proposed method to (optionally) also learn optimal Q-functions and/or policies (\Cref{sec:qrl-impl-extension}). 
    \item On offline \maze tasks, QRL performs well in single-goal and multi-goal evaluations, improving $> 37\%$ over the best baseline and $> 46\%$ over the \dFOURrl handcoded reference controller \cite{fu2020d4rl} (\Cref{sec:maze2d}).
    \item Our learned value functions can be directly used in conjunction with trajectory modeling and planning methods, improving their performances (\Cref{sec:maze2d}).
    \item On online goal-reaching settings, QRL shows up to $4.9\times$ improved sample efficiency and performance in both state-based and imaged-based observations, outperforming baselines including Contrastive RL \citep{eysenbach2022contrastive} and plugging \qmet Q-function models into existing RL algorithms \citep{liu2022metric} (\Cref{sec:online}).
\end{itemize}
% \vspace{-1.7pt}
%!TEX root = ../main.tex
\vspace{-1pt}\section{Value Functions are Quasimetrics}\label{sec:val-fn-qmet}
This section covers the preliminaries on goal-reaching RL settings, value functions, and \qmets. We also present a new result showing an equivalence between the latter two.

\vspace{-2pt}\subsection{Goal-Reaching Reinforcement Learning}

% \myparagraph{Goal-reaching RL.} 
We focus on the goal-reaching RL tasks in the form of   Markov Decision Processes (MDPs): $(\mathcal{S}, \mathcal{A}, P, R)$, where $\mathcal{S}$ is the state space, $\mathcal{A}$ is the action space, $P \colon \mathcal{S} \times \mathcal{A} \rightarrow \Delta(\mathcal{S})$ is the transition function, and $R \colon \mathcal{S} \times \mathcal{S} \rightarrow [R_\mathsf{min}, 0]$ is the reward (cost) function for performing a transition between two states.   $\Delta(A)$ denotes the set of distributions over set $A$.

Given a target state $s_\mathsf{goal} \in \mathcal{S}$, a goal-conditioned agent $\pi(a \given s; s_\mathsf{goal})$ is tasked to reach $s_\mathsf{goal}$ as soon as possible from the current state $s$. 
%
% Formally, until the agent reaches the goal, it receives a constant negative reward of $-1$ for each transition \footnote{This can be relaxed to be any cost function that depends on starting and ending states of a transition. We use a constant here for simplicity.}. The agent $\pi$ aims to maximize the expected total reward given any $s$ and $s_\mathsf{goal}$, which equals the negated number of transitions taken. We call this quantity the (goal-conditioned) value function $V^\pi(s; s_\mathsf{goal})$.
%
Formally, until the agent reaches the goal, it receives a negative reward (cost) $r(s, s')$ for each transition $(s, s')$. The agent $\pi$ aims to maximize the expected total reward given any $s$ and $s_\mathsf{goal}$, which equals the negated total cost. We call this quantity the (goal-conditioned) on-policy value function $V^\pi(s; s_\mathsf{goal})$ for $\pi$.

There exists an optimal policy $\pi^*$ that is universally optimal: \vspace{-11.5pt}\begin{equation}
    \forall s, s_\mathsf{goal}, \quad V^{\pi^*}(s; s_\mathsf{goal}) = \max_\pi V^\pi(s; s_\mathsf{goal}). \vspace{-3pt}
\end{equation}
We thus define the optimal value function $V^* \trieq V^{\pi^*}$.

Similarly, we can define the optimal state-action value function, \ie, Q-function: \begin{equation*}
    Q^*(s, a; s_\mathsf{goal}) \trieq \expect[s'\sim P(s, a)]{R(s, s') + V^*(s'; s_\mathsf{goal})}.
\end{equation*}

\vspace{-7.7pt}\subsection{Value-Quasimetric Equivalence}

Regardless of the underlying MDP, some \emph{fundamental} properties of optimal value $V^*$ always hold.

\myparagraph{Triangle Inequality.} As observed in prior works \citep{wang2022iqe,wang2022learning,liu2022metric,pitis2020inductive,durugkar2021adversarial}, the optimal value $V^*$ 
always 
obeys the triangle inequality %
% (due to its optimality and the Markov property of MDPs): 
(due to optimality and Markov property): 
\vspace{-11.5pt}
\begin{equation}
    \forall s_1, s_2, s_3, \quad V^*(s_1; s_2) + V^*(s_2; s_3) \leq V^*(s_1; s_3). \label{eq:v-tri-ineq}
\end{equation}
Intuitively, $V^*(s_1, s_3)$ is the highest value among all plans from $s_1$ to $s_3$; and $V^*(s_1; s_2) + V^*(s_2; s_3)$ is the highest among all plans from $s_1$ to $s_2$ and then to $s_3$, a more restricted set. Thus, \Cref{eq:v-tri-ineq} holds, and $-V^*$ is just like a metric function on $\mathcal{S}$, except that it may be \emph{asymmetrical}. 

\myparagraph{Quasimetrics} are a generalization of metrics in that they do not require symmetry. For a set $\mathcal{X}$, a \qmet is a function $d \colon \mathcal{X} \times \mathcal{X} \rightarrow \R_{\geq 0}$ such that \begin{align}
    \forall x_1, x_2, x_3,& \quad d(x_1, x_2) + d(x_2, x_3) \geq d(x_1, x_3) \\[-0.03cm]
    \forall x, & \quad d(x, x) = 0.
\end{align}
We use $\mqmet(\mathcal{X})$ to denote all such quasimetrics over $\mathcal{X}$.

\Cref{eq:v-tri-ineq} shows that $-V^* \in \mqmet(\mathcal{S})$. In fact, the other direction also holds: for any $d \in \mqmet(\mathcal{S})$, $-d$ is the optimal value function for some MDP defined on $\mathcal{S}$. 

\begin{theorem}[Value-Quasimetric Equivalence]\label{thm:val-qmet-equiv}
    \begin{align}
        \mqmet(\mathcal{S}) \equiv \{-V^* \colon V^* &\mathrlap{\text{ is the optimal value of}} \notag\\[-0.08cm]
        &\hspace{1em}\text{an MDP on $\mathcal{S}$}\}.\\[-16pt]\notag
    \end{align}%
    Generally, on-policy value $-V^\pi$ may not be a \qmet.%
    % \vspace{-5pt}
\end{theorem}%
All proofs are deferred to the appendix.

\myparagraph{Structure emerges in multi-goal settings.}
The space of quasimetrics is the \emph{exact} function class for goal-reaching RL. In contrast, a specific-goal value function $V^*(\wcdot; s_\mathsf{goal})$ can be any \emph{arbitrary} function $\mathcal{S} \rightarrow \R$. In other words, going from 
% traditional
single-task RL to multi-task RL may be a harder problem, but also has much more structure to utilize (\Cref{fig:structure-multi-goal}). 

% \vspace{-3pt}
\subsection{\Qmet Models and RL}

{Quasimetric Models} refer to parametrized models of quasimetrics $d_\theta \in \mqmet(\mathcal{X})$, where $\theta$ is the parameter to be optimized. Many recent \qmet models are based on neural networks \citep{wang2022iqe,wang2022learning,pitis2020inductive},  can be optimized \wrt any \emph{differentiable} objective, and can potentially generalize to unseen inputs (due to neural networks). Many such models can universally approximate any \qmet and is capable of learning large-scale and complex \qmet structures \citep{wang2022iqe}. 

\myparagraph{An Overview of \Qmet Models.} A \qmet model $d_\theta$ usually consists of (1) a deep encoder mapping inputs in $\mathcal{X}$ to a generic latent space $\R^d$ and (2) a differentiable latent \qmet head $d_\mathsf{latent} \in \mqmet(\R^d)$ that computes the \qmet distance for two input latents. $\theta$ contains both the parameters of the encoder and parameters of the latent head $d_\mathsf{latent}$, if any. Recent works have proposed many choices of $d_\mathsf{latent}$, which have different properties and performances.
%Briefly speaking, a \qmet model on input space $\mathcal{X}$ consists of a deep encoder $g \colon \mathcal{X} \rightarrow \mathbb{R}^d$ and a latent quasimetric head $d_\mathsf{latent} \colon (\mathbb{R}^d \times \mathbb{R}^d) \rightarrow \R^d$, which is a differentiable \qmet function on the learned latent space $\R^d$. Then $g$ and $d_\mathsf{latent}$ 
We refer interested readers to \citet{wang2022iqe} for an in-depth treatment of such models.

% arXiv V1: 
% Previous works explored using such models to parametrize goal-conditioned value functions in standard goal-reaching RL objectives  (which includes enforcing \qmet properties) \citep{wang2022learning,liu2022metric}. 

% arXiv V2:
% \myparagraph{Subtleties of Using \Qmet Models in RL.} While using \Qmet models to parameterize goal-conditioned value functions $V$ in standard RL algorithms seems appealing, this approach has limitations. Standard algorithms employ temporal-difference learning or policy iteration, both of which depend on accurate intermediate results that are not quasimetrics. In fact, on-policy values (policy iteration intermediate results) generally aren't quasimetrics, as shown in \Cref{thm:val-qmet-equiv}. Simply incorporating \qmet models may yield minor benefits or necessitate relaxing constraints and inductive bias. However, we can directly learn $V^*$ using \emph{\qmets}, bypassing these iterative procedures.

\myparagraph{Subtleties of Using \Qmet Models in RL.} It is tempting to parametrize goal-conditioned value functions with \qmet models in standard RL algorithms, which optimizes for $V^* \in \mqmet(\mathcal{S})$. 
% Previous works explored using such models to parametrize goal-conditioned value functions $V$ in standard RL algorithms.
%(which includes enforcing \qmet properties). 
However, these algorithms usually use temporal-difference learning or policy iteration, whose success relies on accurate representation of intermediate results (\eg, on-policy values; \Cref{thm:val-qmet-equiv}) \citep{wang2020statistical,wang2021instabilities}) that are \emph{not} quasimetrics.  Indeed, simply using \qmet models in such algorithms may yield only minor benefits \citep{wang2022learning,wang2022iqe} or require significant relaxations of \qmet inductive bias \citep{liu2022metric}. 

Can we directly learn $V^*$ without those iterative procedures? Fortunately, the answer is \emph{yes}, with the help of \emph{\qmets}.

% In the next section, we present a new RL algorithm that is specifically designed for \qmet models, and is empirically more sample-efficient and better performing \Cref{sec:disc-mnt-car,sec:expr}.
%!TEX root = ../main.tex
% \vspace{-5.5pt}%
% \vspace{-4.5pt}%
\section{Quasimetric Reinforcement Learning}\label{sec:qrl}

\Qmet Reinforcement Learning (QRL) at its core learns the \optC{optimal} goal-conditioned value function $V^*$ that is parametrized by a \qmet model $d_\theta \subset \mqmet(\mathsf{S})$.

Similar to many recent RL works \citep{kumar2019reward,ghosh2019learning,janner2022diffuser,janner2021offline,emmons2021rvs,chen2021decision,paster2022you,yang2022dichotomy}, our method is derived with the assumption that the environment dynamics $P$ are \emph{deterministic}. 
% For simplicity, and as commonly done in goal-reaching RL works \citep{eysenbach2022contrastive,tian2020model,eysenbach2020c,zhang2020generating},  we also assume that each transition has unit cost $1$ (\ie, reward is $-1$), which corresponds to the usual practical objective of reaching the goal as soon as possible. However, this can be potentially relaxed (see \Cref{sec:more-extensions}).
% We simplify discussion by assuming no self-transitions (\ie, a transition from one state to itself). Both our algorithm and theoretical results still work with such transitions (See \Cref{sec:more-extensions}).

Given ways to sample (\eg, from a dataset / replay buffer) 
\begin{align}
    (
    \overbracket{\vphantom{s'}s}^{\mathllap{\textsf{current sta}}\mathrlap{\textsf{te}}}, 
    \underbracket{a}_{\mathllap{\textsf{acti}}\mathrlap{\textsf{on}}}, 
    \overbracket{s'}^{\mathllap{\textsf{ne}}\mathrlap{\textsf{xt state}}}, 
    \underbracket{r}_{\mathllap{\textsf{re}}\mathrlap{\textsf{ward} \leq 0}}
    ) & \sim p_\mathsf{transition}  \tag{transitions} \\[-0.02cm]
    s & \sim p_\mathsf{state}  \tag{random state} \\[-0.02cm]
    s_\mathsf{goal} & \sim p_\mathsf{goal} \tag{random goal}, 
\end{align}
QRL optimizes a \qmet model $d_\theta$ as following:
%
% \begin{equation}
\begin{align}
    % \hspace*{-1cm}
    &\max_{\theta}~\mathbb{E}_{\substack{s\sim p_\mathsf{state}\\g \sim p_\mathsf{goal}}}[ d_\theta(s, g) ] \label{eq:qrl}\\
    % \hspace*{-1cm}
    & \quad \text{subject to } \mathbb{E}_{\substack{(s, a, s', r) \sim p_\mathsf{transition}}}[ \mathtt{relu}(d_\theta(s, s') + r)^2] \leq\epsilon^2,
    % \label{eq:qrl}
    \notag
    % \eqname{\llap{QRL}\vspace{-9pt}}
    % \eqname{\llap{QRL}\vspace{-16pt}} \\[-19pt]\notag
\end{align}%
where $\eps > 0$ is small, and $\mathtt{relu}(x) \trieq \max(x, 0)$ prevents $d_\theta(s,s')$ from exceeding the transition cost $-r \geq 0$.
% \stepcounter{equation}
% \tag{\theequation; \textsf{QRL}}
% \end{equation}

After optimization, we take $-d_\theta$ as our estimate of $V^*$. \Cref{sec:qrl-impl-extension} discusses extensions that learn optimal Q-functions $Q^*$ and policies, making QRL suitable both as a standalone RL method or in conjunction with other RL methods.

\subsection{QRL Learns the Optimal Value Function}\label{sec:qrl-opt-intuition}

\begin{figure*}
    \centering
    % \vspace{-2pt}
    % \includegraphics[scale=0.2605, trim=18 631 20 60, clip]{figures/qrl_push_pull.pdf}%
    \includegraphics[scale=0.263, trim=27 631 20 58, clip]{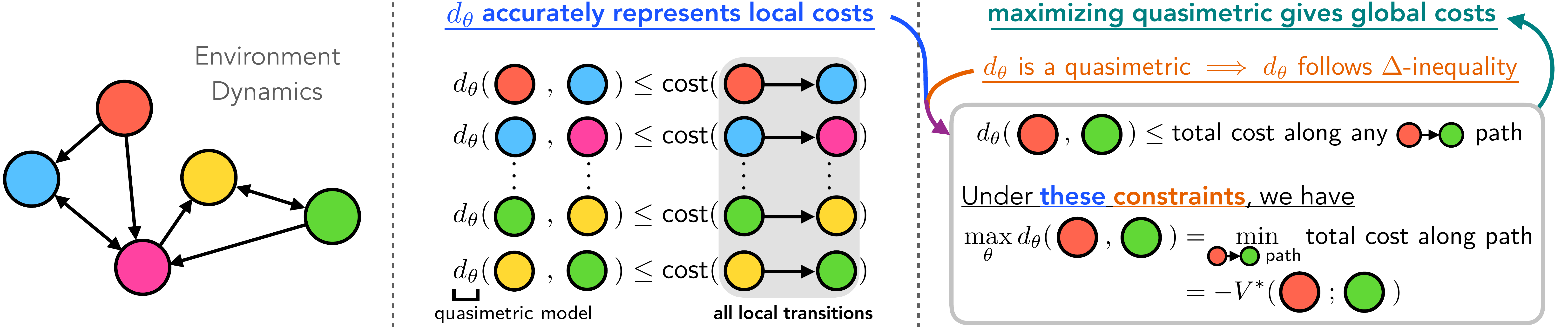}%
    % \vspace{-6.5pt}
    \vspace{-4.5pt}
    \caption{QRL objective finds length of the shortest path connecting two states, \ie, the \optC{optimal value} $V^*$.}%
    \label{fig:qrl-push-pull}%
    % \vspace{-7pt}
    \vspace{-6pt}
\end{figure*}

% Q-Learning objective \citep{sutton2018reinforcement}
By using \qmetC{quasimetric models} $d_\theta$ to parametrize value functions, we inherently satisfy the triangle-inequality constraints. What additional constraints should we add in order to find the optimal value function for a specific MDP? 

\myparagraph{A Physical Analogy.}Consider two objects connected by multiple chains. Each chain is formed by several links. If we \maxiC{pull them apart}, their distance will be limited by the shortest of all chains. Then, simply measuring the distance between the two objects gives the length of that ``optimal'' chain. This argument relies on (1) the \qmetC{triangle inequality} of our Euclidean physical space and (2) that each link of the chains has \presC{a fixed length unaffected by our pulling}.

QRL works  by the  same principles, but in a \qmetC{quasimetric} space that both satisfies the triangle inequality and can capture any asymmetrical MDP dynamics (\Cref{fig:qrl-push-pull}):\begin{itemize}
    \item \textbf{Locally}, we constrain searching of $V^*$ to $d_\theta$'s that are \presC{consistent with local costs, \ie, not overestimating them}: \begin{equation}
        \mathllap{\forall\ \mathsf{transition}}\ (s, a, s', r),\quad    d_\theta(s, s') \mathrlap{{}\leq -r.}  \label{eq:qrl-local-cost}
    \end{equation}
    We ensure this because $d_\theta$ should approximate $-V^*$ and \begin{equation*}
        % d_\theta(s, s') \leq \underbracket{-r}_{\mathclap{\textsf{local cost of this transition}}} = -V^*(s; s'). \label{eq:qrl-local-cost}
        % d_\theta(s, s') \leq -r = -V^*(s; s'). \label{eq:qrl-local-cost}
        % d_\theta(s, s') 
        % & = -V^*(s; s') \\
        -V^*(s; s')
        % & = \mathsf{cost\ of\ shortest\ path\ }s\mathsf{\ to\ }s' \\[-0.15cm]
        \leq \mathsf{cost\ of\ specific\ path\ } s \xrightarrow{\mathsf{action\ } a} s'= -r.
    \end{equation*}
    
    % .  A transition $(s, a, s', r)$ means \begin{equation}
    %     % d_\theta(s, s') \leq \underbracket{-r}_{\mathclap{\textsf{local cost of this transition}}} = -V^*(s; s'). \label{eq:qrl-local-cost}
    %     % d_\theta(s, s') \leq -r = -V^*(s; s'). \label{eq:qrl-local-cost}
    %     % d_\theta(s, s') 
    %     % & = -V^*(s; s') \\
    %     -V^*(s; s')
    %     % & = \mathsf{cost\ of\ shortest\ path\ }s\mathsf{\ to\ }s' \\[-0.15cm]
    %     \leq \mathsf{cost\ of\ path\ } s \xrightarrow{\mathsf{action\ } a} s'= -r.
    % \end{equation}
    % Because $d_\theta$ should approximate $-V^*$, we ensure \begin{equation}
    %     \mathllap{\forall\ \mathsf{transition}} (s, a, s', r),\quad    d_\theta(s, s')  \leq r.  \label{eq:qrl-local-cost}
    % \end{equation}
    
    % Here, \textbf{inequality constraint $d_\theta(s, s') \leq -r$  }  handles the case where each $(s,a)$ pair may induce different costs. If the costs are uniform, this can be an equality constraint.
    
    \item \textbf{Globally}, since $d_\theta$ is a \qmetC{quasimetric} that satisfies the triangle inequality and \Cref{eq:qrl-local-cost}, for \emph{every state $s$ and goal $g$}, {any path $s \rightarrow g$ places a constraint on $d_\theta(s, g)$: } \begin{align}
    & d_\theta(s, g)
        \leq \mathsf{total\ cost\ of\ path\ connecting\ }s\mathsf{\ to\ }g.\notag
    \end{align}
    
    Optimal cost from $s$ to $g$ is given by \maxiC{pulling them apart}: \begin{align}
    % & \hphantom{{}={}} \max_\theta d_\theta(s, g) \notag\\
    \max_\theta d_\theta(s, g)& = \mathsf{cost\ of\ shortest\ path\ connecting\ } 
 s \mathsf{\ to\ } g \notag\\[-0.1cm]
    & = -V^*(s; g). 
\end{align}

\end{itemize}

\textbf{Optimal \qmet $-V^*$} achieves this maxima for all $(s, g)$ pairs. Therefore, we maximize $d_\theta(s, g)$ simultaneously for all $(s, g)$ pairs: \begin{align}
    \theta^* & = \argmax_\theta \mathbb{E}_{\substack{s\sim p_\mathsf{state}\\g \sim p_\mathsf{goal}}}[ d_\theta(s, g) ] \label{eq:qrl-ideal-objective}\\[-0.115cm]
    & \qquad\hspace{10pt} \text{subject to } \forall (s, a, s', r)~\mathsf{transition}, d_\theta(s, s') \leq -r. \notag
\end{align}
This gives exactly the \optC{optimal value}:  \begin{equation}
    d_{\theta^*} (s, g) = -V^* (s; g), \qquad \forall s, g,
\end{equation}
(assuming that $p_\mathsf{state}$ and $p_\mathsf{goal}$ having sufficient coverage).

\textbf{The linear programming characterization of $V^*$ }  \citep{manne1960linear,denardo1970linear}  is similar to \Cref{eq:qrl-ideal-objective}. However, instead of enforcing triangle inequalities via $\size{\mathcal{A}}\size{\mathcal{S}}^2$ constraints, our \qmetC{quasimetric models} automatically satisfy them.

\subsubsection{Theoretical Guarantees}\label{sec:theory}

We now formally state the recovery guarantees for QRL in both the ideal setting (\ie, optimizing over entire $\mqmet(\mathcal{S})$) and the function approximation setting.

The proofs of the following results are mostly formalizations of the ideas above. All proofs are presented in \Cref{sec:proof}.

\begin{theorem}[Exact Recovery]\label{thm:qrl-exact-recovery}
If \Cref{eq:qrl-ideal-objective} optimizes $d_\theta$ over the entire $\mqmet(\mathcal{S})$, %(\ie, $\{d_\theta\}_\theta = \mqmet(\mathcal{S})$),
then for $s \sim p_\mathsf{state}, g \sim p_\mathsf{goal}$, we have $d_{\theta^*}(s, g) = -V^*(s; g)$ almost surely.
\end{theorem}
\vspace{-1pt}

In the more realistic case, we use a \qmet family that is not quite as big as the entire $\mqmet(\mathcal{S})$ but flexible enough to have universal approximation (\eg, IQE \citep{wang2022iqe}). Using a relaxed constraint, we still have a strong guarantee of recovering true $V^*$, ensuring a small error even for $(s, g)$ pairs that are far apart.
\vspace{1.5pt}

\begin{theorem}[Function Approximation; Informal]\label{thm:qrl-func-approx}
\begingroup
Consider a \qmet model family $\{d_\theta\}_\theta$ that is a universal approximator of $\mqmet(\mathcal{S})$ (in terms of $L_\infty$ error). If we solve \Cref{eq:qrl-ideal-objective} with a relaxed constraint, where \begin{equation}
    \forall (s, a, s', r)~\mathsf{transition},\quad \mathtt{relu}(d_\theta(s, s') + r) \leq \eps,\label{eq:qrl-relaxed-constraint}
\end{equation}
for small $\eps >0$. Then, for $s \sim p_\mathsf{state}, g \sim p_\mathsf{goal}$,  we have \begin{equation*}
    \big\lvert {d_{\theta^*}(s, g) + (1 + \eps) V^*(s; g)}  \big\rvert \in [-\sqrt{\eps},0],
\end{equation*}
\ie, $d_{\theta^*}(s, g)$ recovers $-V^*(s; g)$ up to a known scale,
with probability 
$1 - \mathcal{O}(-\sqrt{\eps}\cdot \mathbb{E}[V^*])$.
\endgroup
\end{theorem}
% \vspace{-1pt}

% A general version of \Cref{thm:qrl-func-approx} is proved  in \Cref{sec:proof}.

% \psm{}{
% The above result ensures that the error stays small even for $(s, g)$ pairs that are far apart. We formally state and prove a more general version of \Cref{thm:qrl-func-approx} in the appendix.
% }
% This result ensures that the error stays small even for $(s, g)$ pairs that are far apart. We formally state and prove this result in the appendix.

\begin{figure*}[ht]
    \centering
    \vspace{-3pt}
    \includegraphics[scale=0.254, trim=0 485 0 17, clip]{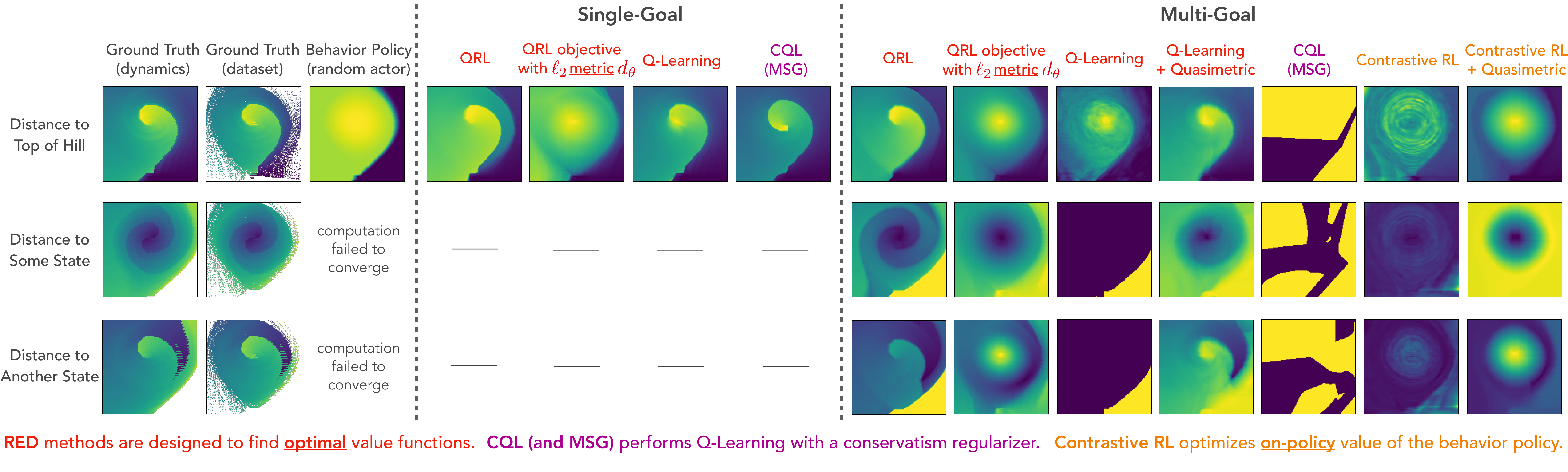}%
    \vspace{-11pt}
    \caption{Learned value functions on offline \mntcar. Each plot shows the estimated values from every state towards a single goal (indicated in the leftmost column) as a $2$-dimensional image (velocity as $x$-axis, position as $y$-axis). \textbf{Left}: Ground truth distances, as well as the (expected) distance for the behavior policy that generated training data. \textbf{Middle}: Learned value functions for single-goal methods. \textbf{Right}: Learned value functions for multi-goal methods. Only QRL accurately recovers the ground truth distance structure in both settings, which crucially relies on the asymmetry of \qmets. Q-learning methods generally fail in multi-goal settings.
    Their learned values, while improved with quasimetric models,  cannot capture the fine details.
    % While adding quasimetrics does improve\psm{, the learned values}{\xspace the learned values, they} still can't capture the fine details.
    Contrastive RL only inaccurately estimates the on-policy values.}
    \label{fig:dmc-valfn}%
    % \vspace{-11pt}
    % \vspace{-10pt}
    \vspace{-5.5pt}
\end{figure*}

\vspace{-3pt}\subsection{A Practical Implementation}

\myparagraph{\Qmet Model.} 
We use Interval Quasimetric Embeddings (IQE; \citet{wang2022iqe}) as our \qmet model family $\smash{\{d_\theta\}_\theta}$.
IQEs have convincing empirical results in learning various \qmet spaces, and enjoy strong approximation guarantees (as needed in \Cref{thm:qrl-func-approx}).

\myparagraph{Constrained Optimization } is done via dual optimization and jointly updating a Lagrange multiplier $\lambda \geq 0$ \citep{eysenbach2021robust}. We use a relaxed constraint that local costs are properly modeled in expectation.

\myparagraph{Stable Maximization of $d_\theta$.} In practice, maximizing
% $ \mathbb{E}_{\substack{s\sim p_\mathsf{state}\\g \sim p_\mathsf{goal}}}[ d_\theta(s, g)]$
$ \mathbb{E}[ d_\theta(s, g)]$
via gradient descent tends to increase the weight norms of the late layers in $d_\theta$. This often leads to slow convergence since $\lambda$ needs to constantly catch up. Therefore, we instead place a smaller weight on distances $d_\theta(s, g)$ that are already large and optimize
% $\mathbb{E}_{\substack{s\sim p_\mathsf{state}\\g \sim p_\mathsf{goal}}}[ \phi(d_\theta(s, g))]$
$\mathbb{E}[ \phi(d_\theta(s, g))]$
, where $\phi$ is a monotonically increasing convex function (\eg, affine-transformed softplus). This is similar to the discount factor in Q-learning, which causes its MSE loss to place less weight on transitions of low value. 

\myparagraph{Full Objective.} Putting everything together, we implement QRL to jointly update $(\theta, \lambda)$ according to
\begin{align}
    % \hspace*{-1cm}
    &
    \min_{\theta} \max_{\lambda \geq 0}
    -\mathbb{E}_{\substack{s\sim p_\mathsf{state}\\g \sim p_\mathsf{goal}}}[ \phi(d^\mathsf{IQE}_\theta(s, g)) ]  + {} \notag\\[-0.25em]
    % \hspace*{-1cm}
    & \qquad 
    \lambda \left(\mathbb{E}_{\substack{(s, a, s', r) \sim p_\mathsf{transition}}}[ \mathtt{relu}(d^\mathsf{IQE}_\theta(s, s') + r)^2] - \epsilon^2 \right).
    \label{eq:qrl-full}
\end{align}
% \vspace{-12pt}

\vspace{-15pt}\subsection{Analyses and Comparisons via Discretized \mntcar}\label{sec:disc-mnt-car}

We empirically analyze QRL and compare to previous works via experiments on the \mntcar environment with a discretized state space. In this environment, the agent observes the location and velocity of a car, and controls it to reach the top of a hill. Due to gravity and velocity, the dynamics are highly asymmetrical. We discretize the $2$-dimensional state space into $160\times 160$ bins so that we can compute the ground truth value functions. We collected an offline dataset by running a uniform random policy, and evaluated the learning result of various methods, including \begin{itemize}
    \item \textbf{QRL}, our method;
    \item Using \textbf{QRL} objective to train a \textbf{symmetrical $\ell_2$ distance} value function;
    \item \textbf{Q-Learning}  with regular unconstrained Q function class;
    \item \textbf{Q-Learning with \qmet} function class;
    \item \textbf{Contrastive RL} \citep{eysenbach2022contrastive}, which uses a contrastive objective but  estimates on-policy values;
    \item \textbf{Contrastive RL} with \qmet function class;
    \item \textbf{Conservative Q-Learning (CQL)} \citep{kumar2020conservative}, which regularizes Q-Learning to reduce over-confidence in out-of-distribution regions;
    \item \textbf{Model Standard-deviation Gradients (MSG)} \citep{ghasemipour2022so}, a state-of-the-art offline RL algorithm using an ensemble of up to $64$ CQL value functions to estimate uncertainty and train policy;
    \item \textbf{Diffuser} \citep{janner2022diffuser}, a representative trajectory modelling methods with goal-conditioned sampling.
\end{itemize}
QRL can be used for both single-goal and multi-goal settings by specifying $p_\mathsf{goal}$. For methods that are not designed for multi-goal settings (MSG and Q-Learning), we use Hindsight Experience Replay (HER; \citet{andrychowicz2017hindsight})  to train the goal-conditioned value functions.

\begin{figure*}
    \centering
    \vspace{-1pt}
    \includegraphics[scale=0.2535, trim=0 170 0 179, clip]{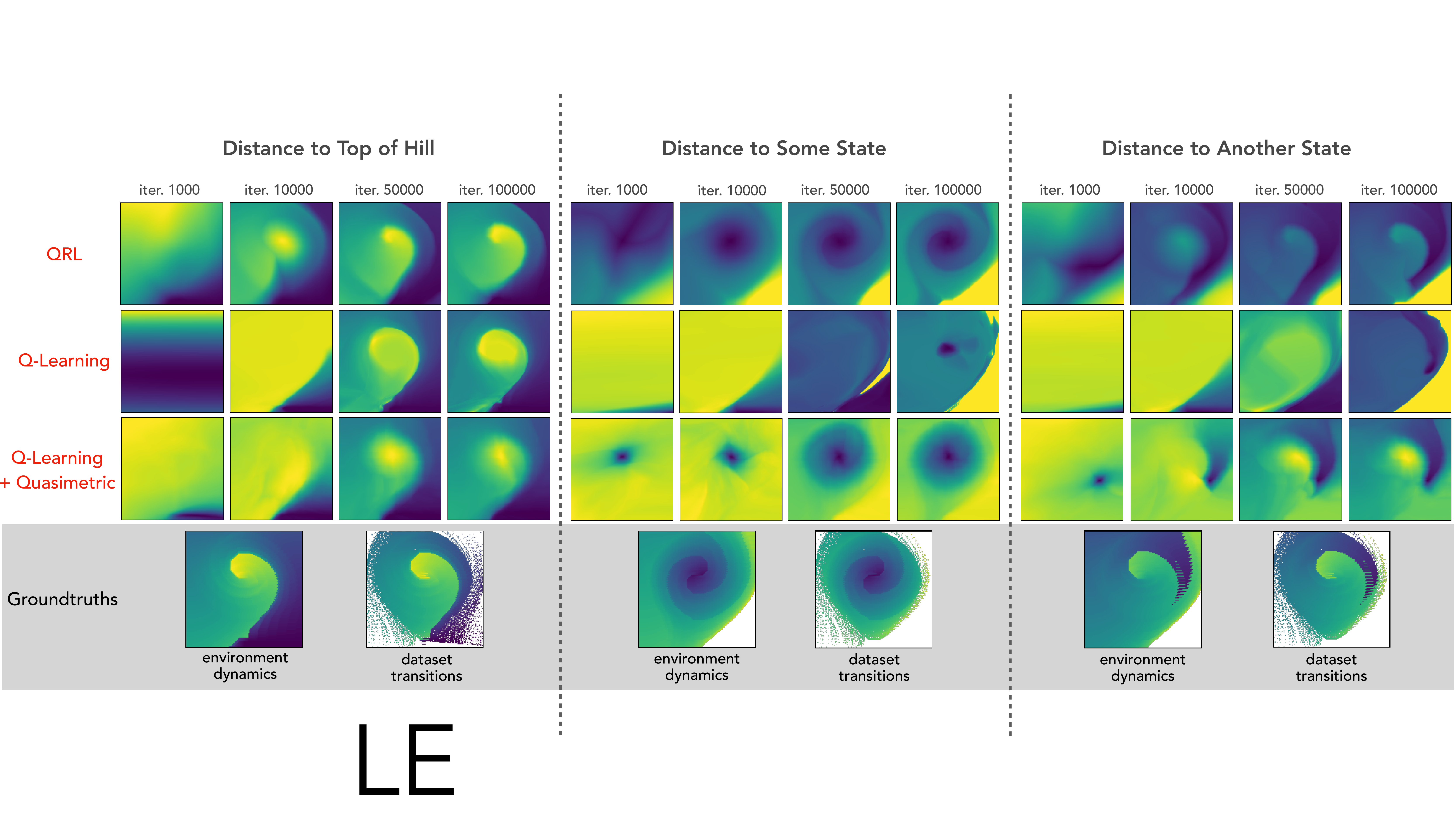}%
    % \vspace{-8pt}
    \vspace{-6pt}
    \caption{Learning dynamics on the offline \mntcar setting. Each plot shows the learned values from every state towards a single goal (indicated at the top) as a $2$-dimensional image (velocity as $x$-axis, position as $y$-axis). Yellow is greater distance (lower value function). Bottom row shows the ground truth distances based on true environment dynamics, and ground truth distances based on transitions appearing in dataset. QRL generally learns the target value function structures much earlier than Q-learning methods. }%
    \label{fig:dmc-traindyn}%
    % \vspace{-10pt}
    \vspace{-6pt}
\end{figure*}
% \FloatBarrier

\vspace{-0.5pt}\myparagraph{Evaluation.} Visually, we compare the learned values against ground truths (\Cref{fig:dmc-traindyn,fig:dmc-valfn}). We test the agents' control performances in both reaching the original goal, top of the hill, as well as $9$ distinct states (\Cref{tab:dismntcar}). A diverse set of goals allows us to evaluate how well the value functions capture the true environment dynamics structure. For QRL and Q-Learning, agents take the action that greedily maximizes the estimated value for simplicity. We describe how to obtain Q-values for QRL later in \Cref{sec:qrl-impl-extension}.

\vspace{-0.5pt}\myparagraph{Q-Learning}
is the standard way to train optimal value functions for such discrete-action space environments. Despite its popularity, many issues have been identified with its temporal-difference training, such as slow convergence \citep{lyle2022learning,fujimoto2022should}.
% } and arguably meaningless loss values \citep{fujimoto2022should}.
\Cref{fig:dmc-traindyn} visualizes the learning dynamics of Q-Learning and QRL, where vanilla Q-Learning indeed learns very slowly. While using a \qmet Q-function helps significantly, QRL still learns the $V^*$ structure much faster, and better captures the true target $V^*$ even after training concludes (\Cref{fig:dmc-valfn}).
% in both single-goal and multi-goal settings and captures the finer details much better.
In planning (\Cref{tab:dismntcar}), vanilla Q-Learning and (Q-Learning based) MSG struggle in multi-goal settings. While Q-Learning with \qmets achieves comparable planning performance with QRL,  the higher-quality $V^*$ estimate from QRL is likely important in more complex environments.
% \Cref{fig:dmc-valfn} shows that, after training concludes, QRL value functions still better capture the true target in both single-goal and multi-goal settings, which are likely important in more complex environments.
Furthermore, with continuous action spaces, Q-Learning requires a jointly learned actor, which (1) reduces to on-policy value learning and (2) can have complicated training dynamics as the actor's on-policy values may not be a \qmet (\Cref{thm:val-qmet-equiv}). QRL is exempt from such issues. In later sections with experiments on online learning in more complex environments, simply using \qmet in traditional value training indeed greatly underperforms QRL (\Cref{sec:online}).
%Additionally, QRL learns more conservative values

\myparagraph{Contrastive RL} uses an arguably similar contrastive objective. However, it samples positive pairs from the same trajectory, and does not enforce exact representation of local costs. Hence, it estimates the on-policy values that generated the data (random actor in this case). Indeed, \Cref{fig:dmc-valfn} shows that the Contrastive RL value functions mostly resemble that of a random actor, and fails to capture the boundaries separating states that have distinct values under \emph{optimal} actors. As shown in \Cref{tab:dismntcar}, this indeed leads to much worse control results.

\myparagraph{Ablations.} We highlight three ablation studies here: \begin{itemize}[itemsep=1pt]
    \item \textbf{Asymmetry.} QRL objective with symmetrical value functions underperforms QRL greatly, suggesting the importance of \qmetC{asymmetry from quasimetrics}.
    \item \textbf{Optimality.} Contrastive RL with \qmet can be seen as a method that uses \qmet to train \emph{on-policy} values. Thus, the learned values fail to capture optimal decision structures. QRL instead enforces \presC{consistency with observed local costs} and \maxiC{maximal spreading of states}, which leads to \optC{optimal values} and better performance.
    \item \textbf{QRL Objective.}
    While Q-Learning with \qmets plans comparably well here , it learns more slowly than QRL (\Cref{fig:dmc-traindyn}) and fails to capture finer value function details (\Cref{fig:dmc-valfn}).
    As discussed above, Q-Learning (with or without \qmets) also has potential issues with complex dynamics and/or continuous action space, while QRL does not have such problems and attain much superior performance in such settings (see later \Cref{sec:online}).
\end{itemize}

\begin{table}[t]%
\vspace{-6.25pt}
%!TEX root = ../main.tex
\centering
\resizebox{
  1.\linewidth
}{!}{%
\renewcommand\normalsize{\small}%
\normalsize%
\centering%
\newlength{\digitl}\settowidth{\digitl}{1}%
\newlength{\basel}\settowidth{\basel}{.1}%
\newlength{\tinydigitl}\settowidth{\tinydigitl}{\tiny1}%
\newlength{\tinybasel}\settowidth{\tinybasel}{\tiny.1}%
\newcommand{\mround}[3][3]{\makebox[\basel+3\digitl][r]{\round{#2}{2}}~{\tiny\textpm~\makebox[\tinybasel+#1\tinydigitl][r]{\round{#3}{2}}}}%
\newcommand{\sround}[2][3]{\makebox[\basel+3\digitl][r]{\round{#2}{2}}}%
\newcommand{\mroundprec}[1]{\round{#1}{2}\%\xspace}%
\newcommand{\bftab}{\fontseries{b}\selectfont}%
\renewcommand{\arraystretch}{1.2}%
\setlength{\tabcolsep}{0.2em} % for the horizontal padding
\setlength\extrarowheight{1pt}%
\setlength{\tabcolsep}{3.5pt}%
\renewcommand{\arraystretch}{1}%
% pots submission
\renewcommand{\arraystretch}{1.01}%
\setlength\extrarowheight{0.15pt}%
\setlength{\tabcolsep}{0.475em} % for the horizontal padding
\hspace*{-0.5em}\begin{tabular}[b]{@{}lccc@{}}
    \toprule
    \multicolumn{1}{c}{\multirow{2}{*}{\psm{}{\vspace{-3pt}{\textbf{Method}}}}}
    &  \multirow{2}{*}{\psm{\vspace{-3pt}Variant}{\vspace{-3pt}\shortstack{\textbf{Method}\\\textbf{Configuration}}}}
    &  \multicolumn{2}{c}{\textbf{Task}} \\[-0.5ex]
    \cmidrule{3-4}

    &
    &  Reach Top of Hill
    &  Reach $9$ States \\[-0.3ex]
    \midrule
    \midrule

    \multirow{2}{*}{\vspace{0pt}\shortstack[l]{QRL} }
    & Single-Goal
    % & \bftab\mround{98.1435196573486}{0.045053138692808}
    & \bftab\mround{97.68766021728516}{0.26373159885406494}  % LE
    & --- \\

    & Multi-Goal
    % & \bftab\mround{95.8021709379247}{0.250294062972744}
    % & \bftab\mround{85.977497990248}{2.44492548847578}
    & \bftab\mround{95.891357421875}{0.5464286208152771}  % LE
    & \bftab\mround{85.5451431274414}{3.5664360523223877}  % LE
    \\[0.0cm]
    \arrayrulecolor{notquitelightgray}\cdashlinelr{1-4}\arrayrulecolor{black}

    \multirow{2}{*}{\vspace{0pt}Q-Learning}
    & --- %Single-Goal
    & \bftab\mround{98.7430535635304}{0.189306012135506}
    & --- \\

    & + Relabel %Multi-Goal$^\dagger$
    & \mround{89.269095263583}{11.6941162165713}
    & \mround{22.0615959872626}{8.71599385937244} \\[0.0cm]
    \arrayrulecolor{notquitelightgray}\cdashlinelr{1-4}\arrayrulecolor{black}

    {\vspace{0pt}Contrastive RL}
    & \multicolumn{1}{c}{---}
    & \mround{83.9081399755776}{8.03959003054404}
    & \mround{53.752953947641}{32.927964407854} \\[0.0cm]
    \arrayrulecolor{notquitelightgray}\cdashlinelr{1-4}\arrayrulecolor{black}

    \multirow{2}{*}{\vspace{0pt}MSG}
    & --- %Single-Goal
    & \bftab\mround{97.4364789925259}{0.21819835570998}
    & --- \\

    & + Relabel %Multi-Goal$^\dagger$
    & \mround{14.3031188302866}{0}
    & \mround{37.8033740935794}{8.20463769421803} \\[0.01cm]
    \arrayrulecolor{notquitelightgray}\cdashlinelr{1-4}\arrayrulecolor{black}

    {\vspace{0pt}Diffuser}
    & \multicolumn{1}{c}{---}
    & \mround{19.7809331010749}{3.02736597248708}
    & \mround{36.4115145524481}{1.43582689181308} \\
    \arrayrulecolor{notquitelightgray}\midrule\arrayrulecolor{black}

    \multirow{2}{*}{\vspace{0pt}\shortstack[l]{QRL Objective\\[-0.2ex]with Symmetric $\ell_2$ Distance} }
    & Single-Goal
    % & \bftab\mround{95.1594265010833}{0.218261736747084}
    & \bftab\mround{95.42462921142578}{0.15999062359333038}
    & --- \\

    & Multi-Goal
    % & \bftab\mround{96.2228983311326}{0.307855468900237}
    & \bftab\mround{96.12943267822266}{0.12171199172735214}  % LE
    % & \mround{73.2260385550671}{1.69272806370617} 
    & \mround{73.2666015625}{0.8376898765563965}  % LE
    \\[0.0cm]
    \arrayrulecolor{notquitelightgray}\cdashlinelr{1-4}\arrayrulecolor{black}

    \multirow{2}{*}{\vspace{0pt}\shortstack[l]{Contrastive RL\\[-0.2ex] + Quasimetric Q-Function}}
    & \multicolumn{1}{c}{\multirow{2}{*}{---}}
    & \multirow{2}{*}{\mround{83.895044357837}{8.72748304931281}}
    & \multirow{2}{*}{\mround{72.2757147936469}{4.63357324608026}} \\

    &
    &
    & \\[0.0cm]
    \arrayrulecolor{notquitelightgray}\cdashlinelr{1-4}\arrayrulecolor{black}

    % \multirow{2}{*}{\vspace{0pt}\shortstack[l]{Q-Learning\\ + Quasimetric Q-Function$^\dagger$}}
    \multirow{2}{*}{\vspace{0pt}\shortstack[l]{Q-Learning\\[-0.2ex] + Quasimetric Q-Function}}
    % & \multicolumn{1}{c}{\multirow{2}{*}{---}}
    & \multicolumn{1}{c}{\multirow{2}{*}{+ Relabel}}
    & \bftab\multirow{2}{*}{\mround{96.33378290781006}{0.3700593792992688}}
    & \bftab\multirow{2}{*}{\mround{85.528737808177}{3.6851571962976024}} \\

    &
    &
    & \\
    \midrule

    {\vspace{0pt}Oracle (Full Dynamics)}
    & \multicolumn{1}{c}{---}
    & \sround{100}
    & \sround{100} \\[0.38ex]

    {\vspace{0pt}Oracle (Dataset Transitions)}
    & \multicolumn{1}{c}{---}
    & \sround{69.2188381401503}
    & \sround{75.8941901} \\

    \bottomrule
\end{tabular}%
}%
\vspace{-7pt}
\caption{\psm{Planning results}{Control results} on \mntcar. Scores  \psm{represent the return}{are normalized returns}   to reach the desired goal within $200$ steps, averaged across all $160\times 160$ starting states. Each row shows \psm{different evaluations}{evaluations} of a method in a specific configuration\psm{}{\xspace with standard deviations from $5$ seeds}.
% \psm{}{\textdagger\ methods apply relabelling with the ground truth reward function.}
We highlight results that are $\geq 95\%$ of the best method.
}%
\label{tab:dismntcar}%
\vspace{-2.5pt}
% \vspace{-19pt}
\end{table}

Compared to existing approaches, QRL efficiently and accurately finds optimal goal-conditioned value functions, showing the importance of both the \qmet structure and the novel learning objective. In the next section, we describe extensions of QRL, followed by more extensive experiments on offline and online goal-reaching benchmarks in \Cref{sec:expr}.

% \subsection{Learning Transition, Q-Function and Policy}\label{sec:qrl-impl-extension}
\vspace{-5.5pt}\subsection{From $V^*$ to $Q^*$ and Policy}\label{sec:qrl-impl-extension}

QRL's optimal value $V^*$ estimate may be used directly in planning to control an agent. A more common approach is to train a policy network \wrt to a Q-function estimate \citep{hafner2019dream}. This section describes simple extensions to QRL that learn the \optC{optimal} Q-function $Q^*$ and a policy.

\myparagraph{Transition and Q-Function Learning.} We augment the \qmet model $d_\theta$ to include an \emph{encoder} $f \colon \mathcal{S} \rightarrow \mathcal{Z}$: \begin{equation}
    d_{\theta = (\theta_1, \theta_2)}(s_0, s_1) \trieq d^z_{\theta_1}(f_{\theta_2}(s_0), f_{\theta_2}(s_1)).
\end{equation}
% which does not affect $d^z_{\theta_1}$'s approximation properties.

Since $d_\theta$ captures $V^*$, finding the Q-function $Q^*(s, a; g)$ only requires knowing the transition result, which we model by a learned latent transition $T \colon \mathcal{Z}\times\mathcal{A} \rightarrow \mathcal{Z}$.
% 
% We thus train an additional latent transition model $T \colon \mathcal{Z}\times\mathcal{A} \rightarrow \mathcal{Z}$ \wrt an \qmet-aware objective. 
In this section, for notation simplicity, we will drop the   $()_{\theta_*}$ subscript, and use $z \trieq f(s)$, $z'\trieq f(s')$, $\hat{z}' \trieq T(z, a)$, and $z_g \trieq  f(g)$.

Once with a well trained $T$, we can estimate $Q^*(s,a;g)$ as\begin{equation}
    % d^z( \underbracket{T(z, a)}_{\mathclap{\textsf{latent transition}}}, f(g)) \approx -Q^*(s,a;g). \label{eq:q-est}
    \smash{d^z( \underbracket{T(z, a)}_{\mathllap{\textsf{latent tran}}\mathrlap{\textsf{sition}}}, z_g)  \underbracket{\vphantom{T(z,a)}{} -r}_{\mathllap{\textsf{trans}}\mathrlap{\textsf{ition cost}}}
     = d^z( \hat{z}', z_g) -r}
    \approx -Q^*(s,a;g).
    \label{eq:q-est}
\end{equation}
$\vphantom{2^{2^{2^{2}}}}$(In our experiments, transition cost $-r$ is a constant, and thus omitted. Generally, $T$ can be extended to estimate $r$.) 

\vspace{2pt}
\emph{Transition loss.} Given transition $(s, a, s')$, we define: \begin{align}
    &\notag\\[-17pt]
    \mathcal{L}_{\mathsf{transition}}(s, a, s'; T, d_\theta)
    &\trieq
    \frac{1}{2} \big( d^z(\hat{z}', z')^2 + d^z(z',\hat{z}')^2 \big),\notag\\[-18.5pt]\notag
\end{align}
which is used to optimize both $d_\theta$ and $T$ in conjunction with the QRL objective in \Cref{eq:qrl-full}.

$\mathcal{L}_{\mathsf{transition}}$ encourages the predicted next latent $\hat{z}'$ to be close to the actual next latent $z'$ \wrt the learned \qmet function $d^z$. This is empirically superior to a simple regression loss on $\mathcal{Z}$, whose scale is meaningless.

More importantly, the \qmet properties allow us to directly relate $\mathcal{L}_{\mathsf{transition}}$ values to Q-function error:

Suppose $d^z(\hat{z}', z')^2 + d^z(z',\hat{z}')^2 \leq \delta^2$, which means \begin{equation}
    d^z(\hat{z}', z') \leq \delta \text{  and  } d^z(z', \hat{z}') \leq \delta.
\end{equation}
\emph{For any} goal $g$\xspace with latent $z_g $, the triangle inequality implies \begin{align*}
    % |{\underbracket{d^z(\hat{z}', z_g)}_{\mathllap{\textsf{estimated }}Q^*(s,a;g)} - \underbracket{d_\theta(s', g)}_{\textsf{estimated }\mathrlap{V^*(s'; g)}}}|
    |{\underbracket{d^z(\hat{z}', z_g)}_{\mathllap{\textsf{estimated}}\textsf{ }Q^*(s,\mathrlap{a;g)}} - \underbracket{d_\theta(s', g)}_{\mathllap{\textsf{estima}}\mathrlap{\textsf{ted }V^*(s'; g)}}}|
    = |d^z(\hat{z}', z_g) - d^z(z', z_g)| 
    &
    \leq \delta.
    % \underbracket{d^z(\hat{z}', z_g)}_{\mathclap{\textsf{estimated $Q^*(s,a;g)$}}} &\leq \hphantom{-}\hspace*{-2pt}d^z(\hat{z}', z) + d^z(z', z_g) \leq \hphantom{-}\delta + \underbracket{d_\theta(s', g)}_{\mathclap{\textsf{estimated $V^*(s'; g)$}}} \\
    % d^z(\hat{z}', z_g) &\geq -d^z(z, \hat{z}') + d^z(z', z_g) \geq -\delta + d_\theta(s', g).
\end{align*}
In other words, if $d_\theta$ accurately estimates $V^*$, our estimated $Q^*(s,a;g)$ has bounded error, \emph{for any} goal $g$, even though we train with a local objective $\mathcal{L}_\mathsf{transition}$. Hence, simply training the transition loss \emph{locally} ensures that Q-function error is bounded \emph{globally}, thanks to using \qmetC{quasimetrics}.

Based on this argument, our theoretical guarantees for recovering $V^*$ (\Cref{thm:qrl-exact-recovery,thm:qrl-func-approx}) can be potentially extended to $Q^*$ and thus to optimal policy. We leave this as future work.

\vspace{2pt}

\myparagraph{Policy Learning.} We train policy $\pi \colon \mathcal{S} \rightarrow \Delta(\mathcal{A})$ to maximize the estimated Q-function (\Cref{eq:q-est}): \begin{equation}
    \min_{\pi} \mathbb{E}_{\substack{s\sim p_\mathsf{state}\\g \sim p_\mathsf{goal}}}[ d^z(T(f(s), a), f(g)) ].
\end{equation}
Additionally, we follow standard RL techniques, training two critic functions and optimizing the policy to maximize rewards from the minimum of them \citep{fujimoto2021minimalist,eysenbach2022contrastive}. In online settings, we also use an adaptive entropy regularizer \citep{haarnoja2018soft}.

%!TEX root = ../main.tex

\section{Related Work}
\vspace{-1.5pt}\myparagraph{Contrastive Approaches to RL.}
As discussed in \Cref{sec:introduction}, our objective bears similarity to those of contrastive approaches. However, we also differ with them in that we rely on (1) \qmetC{quasimetric models}, (2) \presC{consistency with observed local costs}, and (3) \maxiC{maximal spreading of states} to learn the optimal value function.  
Most contrastive methods satisfy none of these properties, and instead pull together states sampled from the same trajectory for capturing on-policy value/information \citep{eysenbach2022contrastive,ma2022vip,sermanet2018time,oord2018representation}. \citet{yang2020plan2vec} ensures exact representation of local cost, but also enforces non-adjacent states to have distance $2$ via a \emph{metric} function, and thus cannot learn optimal values.  Another related line of work trains contrastive models to estimate the alignment between current state and some abstract goal (\eg, text), which are then used as reward for RL training \citep{fan2022minedojo}. Despite the similar goal-reaching setting, their trained model is potentially sensitive to training data, and estimates a density ratio rather than the optimal cost-to-go.

\vspace{-1.25pt}\myparagraph{\Qmet Approaches to RL.}
\citet{micheli2020multi} consider using \qmets for multi-task planning, but does not use models that enforce \qmet properties. \citet{liu2022metric} use \qmet models to parametrize the Q-function, and shows improved performance with DDPG \citep{lillicrap2015ddpg} and HER \citep{andrychowicz2017hindsight} on goal-reaching tasks. These prior works mostly only estimate \emph{on-policy} value functions, and rely on iterative policy improvements to train policies. \citet{zhang2020generating} use a similar \qmet definition, but does not use \qmet models and focuses on hierarchy learning. In contrast, our work utilizes the full \qmet geometry to directly estimate $V^*$ and produce high-quality goal-reaching agents. Additionally, the Wasserstein-1 distance induced by the MDP dynamics is also a \qmet.  \citet{durugkar2021adversarial} utilize its dual form to derive a similar training objective for reward shaping, but essentially employ a different 1-dimensional Euclidean geometry for each goal state and forgo much of the \qmet structure in $V^*$. 

\vspace{-1.25pt}\myparagraph{Metrics and Abstractions in RL.}
Many works explored learning different state-space geometric structures. In particular, bisimulation metric also relates to optimality, but is defined for single tasks where its metric distance bounds the value difference \citet{castro2020scalable,ferns2014bisimulation,zhang2020learning}. Generally speaking, any state-space abstraction can be viewed as a form of distance structure, including state embeddings that are related to value functions \citep{schaul2015universal,bellemare2019geometric}, transition dynamics \citep{mahadevan2007proto,lee2020predictive}, factorized dynamics \citep{fu2021learning,wang2022denoisedmdps}, \etc. While our method also uses an encoder, our focus is to learn a \qmet that directly outputs the optimal value $V^*$ to reach any goal, rather than bounding it for a single task.
%!TEX root = ../main.tex

% \input{figtext/05_experiments/robodesk_joint_pos}
% \input{figtext/05_experiments/robodesk_rew}

\begin{table*}[t]%
\vspace{-1pt}
\centering
\resizebox{
  1.\linewidth
}{!}{%
\centering
\renewcommand\normalsize{\small}%
\normalsize%
\centering%
\let\digitl\relax
\newlength{\digitl}\settowidth{\digitl}{1}
\let\basel\relax
\newlength{\basel}\settowidth{\basel}{.1}
\let\tinydigitl\relax
\newlength{\tinydigitl}\settowidth{\tinydigitl}{\tiny1}
\let\tinybasel\relax
\newlength{\tinybasel}\settowidth{\tinybasel}{\tiny.1}
\newcommand{\mround}[3][3]{\makebox[\basel+3\digitl][r]{\round{#2}{2}}~{\tiny\textpm~\makebox[\tinybasel+#1\tinydigitl][r]{\round{#3}{2}}}}%
\newcommand{\sround}[2][3]{\makebox[\basel+3\digitl][r]{\round{#2}{2}}}%
\newcommand{\mroundprec}[1]{\round{#1}{2}\%\xspace}%
\newcommand{\bftab}{\fontseries{b}\selectfont}%
\renewcommand{\arraystretch}{1.28}%
\setlength{\tabcolsep}{0.22em} % for the horizontal padding
\setlength\extrarowheight{1pt}
\setlength{\tabcolsep}{4pt}
\renewcommand{\arraystretch}{1.25}%
\hspace*{-17.5pt}\begin{tabular}[b]{c@{\hspace*{7pt}}c@{\hspace*{17pt}}ccccccccc}
    \toprule
    &  \multirow{2}{*}{\textbf{Environment}}
    &  \multirow{2}{*}{QRL}
    &  \multirow{2}{*}{Contrastive RL}
    % &  \multirow{2}{*}{\shortstack{MSG\\(\#ensemble $= 4$)}}
    &  \multirow{2}{*}{\shortstack{MSG\\(\#critic $= 64$)}}
    &  \multirow{2}{*}{\shortstack{MSG + HER\\(\#critic $= 64$)}}
    &  \multirow{2}{*}{\shortstack{MPPI with \\GT Dynamics}}
    &  \multirow{2}{*}{\shortstack{MPPI with\\QRL Value}}
    &  \multirow{2}{*}{Diffuser}
    &  \multirow{2}{*}{\shortstack{Diffuser with\\QRL Value Guidance}}
    &  \multirow{2}{*}{\color{gray}\shortstack{Diffuser with\\Handcoded Controller}}
    \\
    % &  \multirow{2}{*}{\shortstack{\dFOURrl Reference\\Controller}} \\%[-0.5ex]

    % &
    &
    &
    &
    &
    &
    &
    &
    &
    &
    & \\\midrule

    \multirow{4}{*}{\vspace{-7pt}Single-Goal}
    &
    \texttt{large}
    & % QRL
% \bftab\mround{185.2617007744397}{28.45882427966159}  % NO LE
\bftab\mround{191.517078828239}{18.27651611406777}  % LE
    & % CRL
\mround{81.65138987616447}{43.79214809198385}
    & % MSG n64
\mround{159.3}{49.4}
    & % MSG-L n64
\mround{59.26371383666992}{46.702842712402344}
    & % MPPI
5.1
    & % MPPI QRL
% \mround{4.669835090637207}{5.3055853843688965}
\mround{19.315349578857422}{22.971527099609375}
    & % Diffuser
\mround{7.976355269557409}{1.5414569161063014}
    & % Diffuser QRL
% \mround{11.334506034851074}{1.484433650970459}
\mround{10.083429982416101}{2.974629708996291}
    & % Diffuser Cheat
\color{gray}\mround{128.13348797186578}{2.5903789139245683}
    % & % d4rl
    % 100
    \\
    
    &
    \texttt{medium}
    & % QRL
% \bftab\mround{148.48406872019982}{46.75364907654691}  % NO LE
\bftab\mround{163.58964656020586}{9.695370264565597}  % LE
    & % CRL
\mround{10.10595625520321}{0.9892025141165419}
    & % MSG n64
\mround{57.0}{17.2}
    & % MSG-L n64
\mround{75.76554870605469}{9.024413108825684}
    & % MPPI
10.2
    & % MPPI QRL
% \mround{60.892303466796875}{40.377410888671875}
\mround{58.06024932861328}{42.78642654418945}
    & % Diffuser
\mround{9.482328010292893}{2.2146122134102}
    & % Diffuser QRL
% \mround{10.523727416992188}{3.260510206222534}
\mround{10.712934231438734}{4.58541982164502}
    & % Diffuser Cheat
\color{gray}\mround{127.63641867857412}{1.4687256627527425}
    % & % d4rl
    % 100
    \\
    
    &
    \texttt{umaze}
    & % QRL
% \mround{47.402362147670445}{23.7174262839025}  % NO LE
\mround{71.71799144989491}{26.21098647907585}  % LE
    & % CRL
\mround{95.10760089848559}{46.22730017235405}
    & % MSG n64
\bftab\mround{101.1}{26.3}
    & % MSG-L n64
\mround{55.640899658203125}{31.819015502929688}
    & % MPPI
33.2
    & % MPPI QRL
% \mround{45.880733489990234}{9.321934700012207}
\mround{74.84674835205078}{21.304964065551758}
    & % Diffuser
\mround{44.0286935729295}{2.24621217256269}
    & % Diffuser QRL
% \mround{42.19404220581055}{4.22982120513916}
\mround{42.30128251575972}{3.8693173999210875}
    & % Diffuser Cheat
\color{gray}\mround{113.91203535975652}{3.266293021302733}
    % & % d4rl
    % 100
    \\\cmidrule{2-11}
    
    % \multicolumn{2}{c}{\textbf{Single-Goal Average}}
    &
    \textbf{Average}
    & % QRL
% \bftab\sround{127.0493772}  % NO LE
\bftab\sround{142.27490561277992}  % LE
    & % CRL
    \sround{62.28831568}
    & % MSG n64
    \sround{105.8}
    & % MSG-L n64
    \sround{63.55672073}
    & % MPPI
\sround{16.16666667}
    & % MPPI QRL
% \sround{37.14762402}
\sround{50.74078369140625}
    & % Diffuser
\sround{20.49579228}
    & % Diffuser QRL
% \sround{21.35075855}
\sround{21.032548909871522}
    & % Diffuser Cheat
\color{gray}\sround{123.227314}
    % & % d4rl
    % 100
    \\[0.1ex]\midrule

    \multirow{4}{*}{\vspace{-7pt}Multi-Goal}
    &
    \texttt{large}
    & % QRL
% \bftab\mround{199.19338546148373}{4.066102449523602}  % NO LE
\bftab\mround{187.71297093045004}{7.620717460809017}  % LE
    & % CRL
\mround{172.63945527329864}{5.1278449218848285}
    & % MSG n64
---
    & % MSG-L n64
\mround{44.566566467285156}{25.30406379699707}
    & % MPPI
8
    & % MPPI QRL
% \mround{54.04167938232422}{7.47212028503418}
\mround{37.73055648803711}{16.667526245117188}
    & % Diffuser
\mround{13.089902353249277}{0.9958638457454098}
    & % Diffuser QRL
% \mround{21.783082962036133}{2.862971544265747}
\mround{21.25855812039358}{2.950895757729535}
    & % Diffuser Cheat
\color{gray}\mround{146.94376894010253}{2.5005206039895898}
    % & % d4rl
    % 100
    \\
    
    &
    \texttt{medium}
    & % QRL
% \bftab\mround{161.90796942405203}{8.09967207330522}  % NO LE
\bftab\mround{150.51464466812988}{3.773475440502487}  % LE
    & % CRL
\mround{137.00749262090363}{6.263057821471228}
    & % MSG n64
---
    & % MSG-L n64
\mround{99.75630187988281}{9.831496238708496}
    & % MPPI
15.4
    & % MPPI QRL
% \mround{71.23817443847656}{6.694705009460449}
\mround{56.7864990234375}{7.66044282913208}
    & % Diffuser
\mround{19.20986906834178}{3.5592153930164714}
    & % Diffuser QRL
% \mround{33.6751708984375}{2.8184902667999268}
\mround{33.39060016650269}{2.7825231494029152}
    & % Diffuser Cheat
\color{gray}\mround{119.97048361462195}{1.2246455934857097}
    % & % d4rl
    % 100
    \\
    
    &
    \texttt{umaze}
    & % QRL
% \mround{134.112020868053}{12.562786132684867}  % NO LE
\bftab\mround{150.60213028041446}{5.322403596419383}  % LE
    & % CRL
% \bftab
\mround{142.42880950655749}{11.989339878072997}
    & % MSG n64
---
    & % MSG-L n64
\mround{27.897979736328125}{10.38830852508545}
    & % MPPI
41.2
    & % MPPI QRL
% \mround{84.72283935546875}{7.694775104522705}
\mround{87.49366760253906}{9.719127655029297}
    & % Diffuser
\mround{56.21911455691615}{3.89816426611298}
    & % Diffuser QRL
% \mround{69.49061584472656}{3.851199150085449}
\mround{69.9586986450257}{2.3887805095746732}
    & % Diffuser Cheat
\color{gray}\mround{128.5298166799507}{0.9983784596147208}
    % & % d4rl
    % 100
    \\\cmidrule{2-11}
    
    % \multicolumn{2}{c}{\textbf{Single-Goal Average}}
    &
    \textbf{Average}
    & % QRL
% \bftab\sround{165.0711253}  % NO LE
\bftab\sround{162.94324862633144}
    & % CRL
    \sround{150.6919191}
    & % MSG n64
---
    & % MSG-L n64
    \sround{57.40694936}
    & % MPPI
\sround{21.53333333}
    & % MPPI QRL
% \sround{70.00089773}
\sround{60.67024230957031}
    & % Diffuser
\sround{29.50629533}
    & % Diffuser QRL
% \sround{41.64962324}
\sround{41.53595231064067}
    & % Diffuser Cheat
\color{gray}\sround{131.8146897}
    % & % d4rl
    % 100
    \\[0.1ex]
    \bottomrule
\end{tabular}%
}%
\vspace{-5pt}
\caption{Planning results on \maze. Scores represent average normalized episode return, where $100$ represents comparable performance with the \dFOURrl reference handcoded controller. Each column show evaluations of the same method configuration. \Eg, we train goal-reaching QRL agents and evaluate them in both single-goal and multi-goal settings. We highlight results that are $\geq 95\%$ of the best method. \psm{QRL agents significantly outperform baselines in both evaluations.}{In both evaluations, QRL agents significantly outperform baselines, including MSG + HER with the ground truth reward function, and MPPI with the ground truth environment dynamics.} QRL value functions can also be used with planning methods (MPPI) or trajectory sampling methods (Diffuser), and improve their performances.
MPPI with GT Dynamics scores are copied from \citet{janner2022diffuser}.
}%
\label{tab:maze2d}%
\vspace{-5pt}
\end{table*}

\begin{figure*}[ht]
    \centering
    \vspace{-2pt}
    \includegraphics[scale=0.417, trim=5 0 0 0, clip]{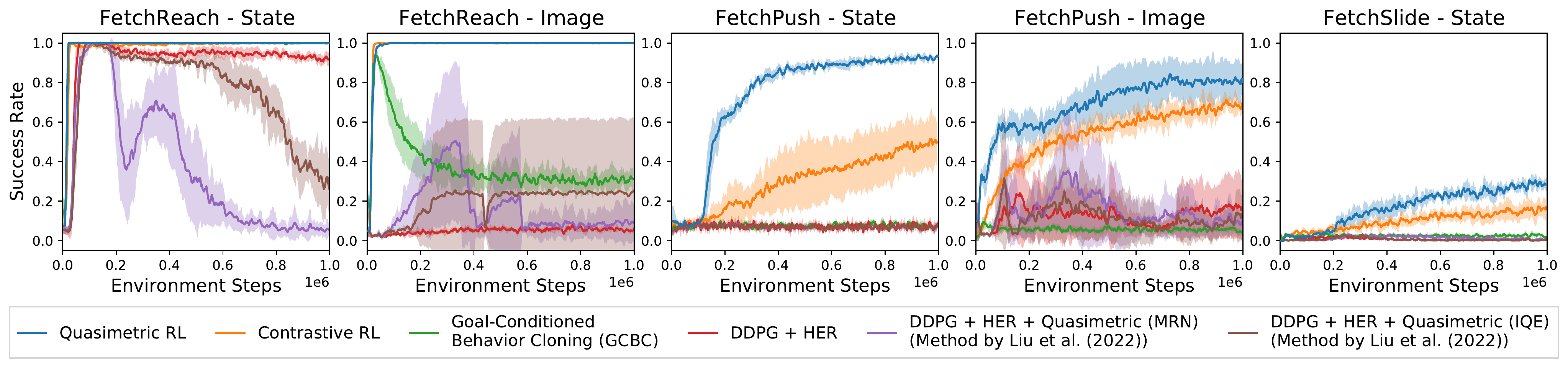}%
    \vspace{-8.5pt}
    \caption{Online learning performance on GCRL benchmarks. No method has access to ground truth reward function. QRL learns faster and better than the baseline methods across all environments for both state-based and image-based observations. %We are unable to get any method to learn on \texttt{FetchSlide} with image observation.
    }
    \label{fig:online}%
    % \vspace{-3.75pt}
\end{figure*}

% \vspace{-4pt}
\section{Benchmark Experiments}\label{sec:expr}

We evaluate QRL learned policies on standard goal-reaching benchmarks in both offline and online settings. All results show means and standard deviations from $5$ seeds. See \Cref{sec:expr-details} for all experiment details.

\subsection{Offline \postsubmission{Multi-Goal}{Goal-Reaching} \dFOURrl \maze}\label{sec:maze2d}

Following Diffuser \citep{janner2022diffuser}, we use  \maze environments from \dFOURrl \citep{fu2020d4rl}, and evaluate the learned policies' performance in (1) reaching the original fixed single goal defined in \dFOURrl as well as (2) reaching goals randomly sampled from the state space. 
Similar to many offline works (\eg, Contrastive RL \citep{eysenbach2022contrastive}), we adopt an additional behavior cloning loss for QRL policy optimization in this offline setting.

\myparagraph{QRL is a strong method for offline goal-reaching RL.}
In \Cref{tab:maze2d}, QRL significantly outperforms all baselines in both single-goal and multi-goal settings. MSG uses a $64$-critic ensemble and is computationally expensive. With only $2$ critics, QRL outperforms MSG by $20\%$ on single-goal tasks and $188\%$ on multi-goal tasks. The Diffuser original paper reported results from a handcoded controller with sampled states as input waypoints. We also report planning using Diffuser's sampled actions, which attains a much worse result. Regardless, QRL outperforms both Diffuser settings, without using any external information/controller. Compared with Contrastive RL, QRL again sees a big improvement, especially in the single-goal setting. Since the dataset is not generated by agents trying to reach that goal, the on-policy values estimated by Contrastive RL are likely much worse than the optimal values from QRL.

\myparagraph{QRL learned value function improves planning and trajectory sampling methods.}
Given the high quality of QRL value functions, we can use it to improve other methods. MPPI \citep{williams2015model} is a model-based planning method. When planning with QRL Q-function, MPPI greatly improves over using ground truth dynamics.  We also experiment using QRL Q-function to guide Diffuser's goal-conditioned sampling, and obtain consistent and non-trivial improvements, especially in multi-goal settings.

\vspace{-1.25pt}\subsection{Online Goal-Reaching RL}\label{sec:online}

Following Contrastive RL \citep{eysenbach2022contrastive} and Metric Residual Networks (MRN; \citet{liu2022metric}), we use the \texttt{Fetch} robot environments from the GCRL benchmark \citep{plappert2018multi}, where we experiment with both state-based observation as well as image-based observation. 

\myparagraph{QRL quickly achieves high performance in online RL.} Across all environments, QRL exhibits strong sample-efficiency, and learns the task much faster than the alternatives. Only QRL and Contrastive RL learn in the two more challenging state-based settings, \texttt{FetchPush} and \texttt{FetchSlide}. Compared to Contrastive RL, QRL has $4.9 \times$ sample efficiency on state-based \texttt{FetchPush} and $2.7\times$ sample efficiency on state-based \texttt{FetchSlide}. 
Strictly speaking, image-based observation only contains partial information of the true state, and thus has stochastic dynamics, which violates the assumption of QRL. However, QRL still shows strong performance on image-based settings, suggesting that QRL can potentially also be useful in other partially observable and/or stochastic environments.

\myparagraph{QRL outperforms Q-Learning with \qmet models in complex environments.}
Following the approach by \citet{liu2022metric}, we train standard DDPG \citep{lillicrap2015ddpg} with relabelling and a \qmet model Q-function. Essentially, this jointly optimizes a \qmet Q-function with Q-Learning and a deterministic policy \wrt the Q-function.  While similar approaches worked well on the simple \mntcar environment (\Cref{sec:disc-mnt-car}), they fail miserably here on more complex continuous-control settings, as Q-Learning must estimate \emph{on-policy} Q-function that may not be a \qmet (\Cref{thm:val-qmet-equiv}). DDPG with \qmets are the slowest to learn on state-based \texttt{FetchReach}, and generally are among the least-performing methods. The same pattern holds for two different \qmet models: IQE and MRN (proposed also by \citet{liu2022metric}). In comparison, QRL (which also uses IQE in our implementation) quickly learns the tasks. QRL is more general and scales far better than simply using Q-Learning with \qmets.
%!TEX root = ../main.tex

% \vspace{-2pt}
\section{Implications}

In this work, we introduce a novel RL algorithm, QRL, that utilizes the equivalence between optimal value functions and \qmets. In contrast to most RL algorithms that optimize generic function classes, QRL is designed for using \qmetC{quasimetric models} to parametrize value functions. Combining \qmet models with an objective that \presC{captures local distances} and \maxiC{maximally spreads out states} (\Cref{sec:qrl-opt-intuition}), QRL provably recovers the \optC{optimal} value function (\Cref{sec:theory}) without temporal-difference or policy iteration, making it distinct from many prior approaches.

From thorough analyses on \mntcar, we empirically confirm the importance of different components in QRL, and observe that QRL can learn value functions faster and better than alternatives (\Cref{sec:disc-mnt-car}). Our experiments on additional benchmarks echo these findings, showing better control results in both online and offline settings (\Cref{sec:expr}). QRL can also be used to directly improve other RL methods, and demonstrates strong sample efficiency in online settings. 

These QRL results highlight the usefulness of \qmets in RL, as well as the benefit of incorporating \qmet structures into \emph{designing} RL algorithms. 

Below we summarize several exciting future directions.

\myparagraph{QRL as Representation and World Model Learning.} QRL can be also viewed as learning a decision-aware representation (via encoder $f$) and a latent world model (via latent dynamics $T$). In this work, for fair comparison, we did not utilize such properties much. However, combining QRL with techniques from these areas (\eg, estimating multi-step return, auxiliary loss training) may yield even stronger performances and/or more general QRL variants (\eg, better support for partial observability and stochasticity).  

\myparagraph{\Qmet Structures in Searching and Exploration.} QRL results show that \qmets can flexibly model distinct environments and greatly boost sample efficiency. Such learned (asymmetrical) state-space distances potentially have further uses in long-range planning and exploration. A \presC{locally distance-preserving} \qmet is always a \emph{consistent and admissible} heuristic \citep{pearl1984heuristics}, which guarantees optimality in search algorithms like A* \citep{hart1968formal}. Perhaps such exploration ideas may be incorporated in a \qmet-aware actor, or even for solvers of general searching and planning problems. 
% Additionally, one may use uncertainty/errors in learned \qmet distances/dynamics to derive new intrinsic exploration methods.

\myparagraph{Better Exploration for Structure Learning.} In our and most RL works, online exploration is done via noisy actions from the learned policy. Arguably, if an agent is learning the structure of the environment, it should instead smartly and actively probe the environment to improve its current estimate. Consider QRL as an example. If current \qmet estimate $d_\theta(s_0, s_1)$ is small but no short path connecting $s_0$ to $s_1$ was observed, the agent should test if they are actually close \wrt the dynamics. Additionally, one may use uncertainty/errors in learned \qmet distances/dynamics to derive new intrinsic exploration methods.
Such advanced exploration may speed up learning  the geometric structures of the world, and thus better generalist agents.

\myparagraph{More \Qmet-Aware RL Algorithms.} To our best knowledge, QRL is the first RL method designed for \qmet models. We hope the strong performance of QRL can inspire more work on RL algorithms that are aware of \qmet and/or other geometric structures in RL.

\section*{Acknowledgements}
TW's contribution in this material is based on work that is partially funded by Google. TW was supported in part by ONR MURI grant N00014-22-1-2740.

\clearpage\newpage
\bibliography{reference}
\bibliographystyle{icml2023}

%%%%%%%%%%%%%%%%%%%%%%%%%%%%%%%%%%%%%%%%%%%%%%%%%%%%%%%%%%%%%%%%%%%%%%%%%%%%%%%
%%%%%%%%%%%%%%%%%%%%%%%%%%%%%%%%%%%%%%%%%%%%%%%%%%%%%%%%%%%%%%%%%%%%%%%%%%%%%%%
% APPENDIX
%%%%%%%%%%%%%%%%%%%%%%%%%%%%%%%%%%%%%%%%%%%%%%%%%%%%%%%%%%%%%%%%%%%%%%%%%%%%%%%
%%%%%%%%%%%%%%%%%%%%%%%%%%%%%%%%%%%%%%%%%%%%%%%%%%%%%%%%%%%%%%%%%%%%%%%%%%%%%%%
\newpage
\appendix
\onecolumn
\section{Discussions and Generalizations of QRL}\label{sec:more-extensions}

\myparagraph{Self-Transitions} are in fact already handled by the QRL objective presented in this paper (\Cref{eq:qrl,eq:qrl-full}). For any state $s$ (with or without self-transition), we have $V^*(s; s) = 0$, since the optimal cost to first reach $s$ start from $s$ is given by the empty trajectory. 
This is naturally enforced by our value function model $-d_\theta$, since it is enforced to be a \qmet. For self-transitions $(s, a, s, r)$ in the training data (where $r \leq 0$ is the reward), their contribution to the constraint loss term will always be $\mathtt{relu}(d_\theta(s, s) + r)^2 = \mathtt{relu}(0 + r)^2  = \mathtt{relu}(r)^2  = 0^2 = 0$. Therefore, the constraints are inherently satisfied for self-transitions.
So our theoretical results from \Cref{sec:theory} also hold for such cases.

\myparagraph{Constant Costs.} In many cases (and most goal-reaching benchmarks), each environment transition has a fixed constant cost $C$. In other words, the task is to reach the given goal as quickly as possible. Then, in the QRL constrained optimization, we can drop the $\mathtt{relu}(\cdot)$ and essentially, since we know that $-V^*(s; s') = C$ for sure, and thus we should have $d_\theta(s, s') = C$. Technically speaking, the $\mathtt{relu}(\cdot)$ formulation should be able to find the same solution. In our experience, even when it is known that the transition cost is constant, adding this information in the objective, \ie, removing $\mathtt{relu}(\cdot)$, does not significantly change the results.

% \myparagraph{Non-constant Costs.}
% The main problem with non-constant cost arise when two transitions connecting the same pair of states have different costs. To address this, one can potentially modify the QRL objective to only enforce that estimated local distances should be \emph{no larger than} the observed cost. This way the maximization term will still find the true minimum cost, theoretically speaking.

\myparagraph{General Goals (Sets of States).} We can easily extend QRL to general goals, which are sets of states. Let $G \subset \mathcal{S}$ be such a general goal. We augment our models to operate not just on $\mathcal{S}$, but on $\mathcal{S} \bigcup \{G\}$ (which can be simply achieved by, \eg, adding an indicator dimension). When we encounter transition that ends within some $s' \in G$, we simultaneously add a transition $(s', G)$ to the dataset.

\section{Proofs}\label{sec:proof}

\subsection{\Cref{thm:val-qmet-equiv}: \nameref{thm:val-qmet-equiv}}

\begin{proof}[Proof of \Cref{thm:val-qmet-equiv}]
We have shown already $-V^* \in \mqmet(\mathcal{S})$. (See also Proposition~A.4 of \citet{wang2022learning}.)

For any $d \in \mqmet(\mathcal{S})$, define \begin{align*}
    \mathcal{A} & \trieq \mathcal{S} \\
    P(s, s_\mathsf{act}) & \trieq \delta_{s_\mathsf{act}} \tag{$\delta_x$ is the Dirac measure at $x$}\\
    R(s, s') & \trieq -d(s, s').
\end{align*}
Then the optimal value of $(\mathcal{S}, \mathcal{A}, P, R)$  is $-d$, regardless of discounting factor (if any).

For the on-policy values, consider action space $\mathcal{A} = \{a_\mathsf{self}, a_\mathsf{next}\}$. Assume that state-space $\size{\mathcal{S}} > 2$. Let $s_1$, $s_2$, $s_3$ be three distinct states in $\mathcal{S}$, all transitions have reward $-1$, and \begin{align}
    P(s_1, a_\mathsf{self}) & = \delta_{s_1} \notag\\
    P(s_2, a_\mathsf{self}) & = \delta_{s_2} \notag\\
    P(s_3, a_\mathsf{self}) & = \delta_{s_3}  \tag{$a_\mathsf{self}$ is always a self-transition}\\
    P(s_1, a_\mathsf{next}) & = \delta_{s_2} \notag\\
    P(s_2, a_\mathsf{next}) & = \delta_{s_3} \notag\\
    P(s_3, a_\mathsf{next}) & = \delta_{s_1}  \tag{$a_\mathsf{next}$ goes to the next state cyclically}\\
    \pi(s_1; s_2) & = \delta_{a_\mathsf{next}}\hspace{-5em}\tag{$\pi$ always takes $a_\mathsf{next}$ when tasked to go to $s_2$ from $s_1$} \\
    \pi(s_2; s_3) & = \delta_{a_\mathsf{next}} \notag\\
    \pi(s_1; s_2) & = \delta_{a_\mathsf{self}}. \notag
\end{align}
So \begin{align*}
    -V^\pi(s_1; s_2) = -V^{\pi}(s_2; s_3) & = 1 \\
    -V^\pi(s_3; s_1) & = \infty,
\end{align*}
violating triangle-inequality.
\end{proof}

\subsection{\Cref{thm:qrl-exact-recovery}: \nameref{thm:qrl-exact-recovery}}

\begin{proof}[Proof of \Cref{thm:qrl-exact-recovery}]
Since the transition dynamics is deterministic, we can say that states $s_0$ is locally connected to state $s_1$  if $\exists a \in \mathcal{A}$ such that $P(s_1 \given s_0, a) = 1$. We say a path $(s^\mathsf{path}_{0}, s^\mathsf{path}_{1}, s^\mathsf{path}_{2}, \dots, s^\mathsf{path}_{T})$ connects $s_0$ to $s_1$ if \begin{align*}
    s^\mathsf{path}_{0} & = s_0 \\
    s^\mathsf{path}_{i} & \text{ is locally connected to } s^\mathsf{path}_{i + 1}, \quad \forall i \in \{0, 1, \dots, T - 1\} \\
    s^\mathsf{path}_{T} & = s_1.
\end{align*}
And we say the total cost of this path is the total rewards over all $T-1$ transitions, \ie, $T - 1$.

From the definition of $V^*$ and \Cref{thm:val-qmet-equiv}, We know that, \begin{align}
    -V^* & \in \mqmet(\mathcal{S}) \notag \\
    -V^*(s; g) & = \textsf{total cost of shortest path connecting $s$ to $g$}, \qquad \forall s, g. \notag
\end{align}

Therefore, the constraints stated in \Cref{eq:qrl-ideal-objective} is feasible. The rest of this proof focuses only on $d_\theta$'s that satisfy the constraints, which includes $-V^*$.

Due to triangle inequality, we have $\forall s, g$, \begin{equation}
    d_\theta(s, g) \leq \textsf{total cost of shortest path connecting $s$ to $g$} = -V^*(s;g).
\end{equation}

Therefore, \begin{equation}
    \mathbb{E}_{s, g} [d_\theta(s, g)] \leq \mathbb{E}_{s, g} [- V^*(s; g)],
\end{equation}
with equality iff $d_\theta(s, g) = -V^*(s; g)$ almost surely.

Hence, $d_{\theta^*}(s, g) = -V^*(s; g)$ almost surely.

\end{proof}

\subsection{\Cref{thm:qrl-func-approx}: Function Approximation}

We first state the more general and formal version of \Cref{thm:qrl-func-approx}.

\begin{theorem}[Function Approximation; General; Formal]\label{thm:qrl-func-approx-formal}
Assume that $\mathcal{S}$ is \ul{compact} and $V^*$ is \ul{continuous}. 

Consider a \qmet model family that is a universal approximator of $\mqmet(\mathcal{S})$ in terms of $L_\infty$ error (\eg, IQE \citep{wang2022iqe} and MRN \citep{liu2022metric}). Concretely, this means that $\forall \eps > 0$, we can have $\{d^{(\eps)}_\theta\}_\theta$ such that, there exists some $\theta$ where \begin{equation}
   \forall s_0, s_1 \in \mathcal{S},\quad \abs{d^{(\eps)}_\theta(s_0, s_1) + V(s_0; s_1)} \leq \eps.
\end{equation}

Now for some small $\eps > 0$, consider solving  \Cref{eq:qrl-ideal-objective} over $\{d^{(\eps/2)}_\theta\}_\theta$ with the relaxed constraint that \begin{equation}
    \forall (s, a, s', r)~\mathsf{transition},\quad \mathtt{relu}({d^{(\eps/2)}_\theta(s, s') + r}) \leq \eps,
\end{equation}
then for $s \sim p_\mathsf{state}, g \sim p_\mathsf{goal}$, and for all $ \delta > 0$, we have \begin{equation*}
    \abs{d_{\theta^*}(s, g) + (1 + \eps) V^*(s; g)} \in [-\delta,0],
\end{equation*}
with probability $1 - \mathcal{O}\left(\frac{\eps}{\delta} \cdot (- \mathbb{E}[V^*])\right)$.

As a special case with $\delta = \sqrt{\eps}$, we have \begin{equation}
    \mathbb{P}\bigg [ 
    \big\lvert {d_{\theta^*}(s, g) + (1 + \eps) V^*(s; g)}  \big\rvert \in [-\sqrt{\eps},0]
    \bigg ] = 1 - \mathcal{O}\left( - \sqrt{\eps} \cdot \mathbb{E}[V^*]\right),
\end{equation}
which is exactly \Cref{thm:qrl-func-approx}.
    
\end{theorem}

Note that the compactness and continuity assumptions ensure that $V^*$ is bounded. We start by proving a lemma.

\begin{lemma} \label{lemma:relaxed-qrl-lb}
With the assumptions of \Cref{thm:qrl-func-approx-formal}, there exists a $d^{(\eps/2)}_{\theta^\dagger}$ that satisfies the constraint with 
\begin{equation}
    \mathllap{\forall s, g,\quad}
    d^{(\eps/2)}_{\theta^\dagger}(s, g) \geq - V^*(s; g).
\end{equation}
\end{lemma}
\begin{proof}[Proof of \Cref{lemma:relaxed-qrl-lb}]

Let the underlying MDP of $V^*$ be $\mathcal{M} = (\mathcal{S}, \mathcal{A}, P, R)$. Consider another MDP $\tilde{\mathcal{M}} =  (\mathcal{S}, \mathcal{A}, P, R - \frac{\eps}{2})$, with optimal goal-reaching value function $\tilde{V}^* \in \mqmet(\mathcal{S})$.

Obviously, transitions $(s, a, s', r)$ in $\mathcal{M}$ bijectively correspond to transitions $(s, a, s', r -\frac{\eps}{2})$ in $\tilde{\mathcal{M}}$. 

For any $s$ and $g$, let $s \rightarrow s_1 \rightarrow s_2 \rightarrow \dots \rightarrow s_{n-1} \rightarrow g$ be the shortest path connecting $s$ to $g$ in $\tilde{\mathcal{M}}$ via $n$ transitions. \begin{align}
    -\tilde{V}^*(s; g) 
    & = {\textsf{total cost of $s \rightarrow s_1 \rightarrow s_2 \rightarrow \dots \rightarrow s_{n-1} \rightarrow g$ according to $R - \frac{\eps}{2}$ as reward}} \\
    & = \frac{n \cdot \eps}{2} + {\textsf{total cost of $s \rightarrow s_1 \rightarrow s_2 \rightarrow \dots \rightarrow s_{n-1} \rightarrow g$ according to $R$ as reward}} \\
    & \geq \frac{n \cdot \eps}{2} + {\textsf{total cost of shortest path connecting $s$ to $g$ in $\mathcal{M}$ according to $R$ as reward}} \\
    & = \frac{n \cdot \eps}{2} -V^*(s; g).
\end{align}

Since $n > 0$ iff $s \neq g$, we have \begin{equation}
    -\tilde{V}^*(s; g) \geq \frac{\eps}{2} \cdot \mathbbm{1}_{s \neq g} -V^*(s; g), \qquad \forall s, g. \label{eq:lemma:v-vtilde}
\end{equation}

By  universal approximation, there exists $d^{(\eps/2)}_{\theta^\dagger}$  such that \begin{equation}
    \forall s, g,\quad \abs{ d^{(\eps/2)}_{\theta^\dagger}(s, g) + \tilde{V}^*(s; g) } \leq \frac{\eps}{2}. \label{eq:lemma:ua}
\end{equation}

In particular, \begin{itemize}
    \item for $s \neq g$, by \Cref{eq:lemma:v-vtilde,eq:lemma:ua}, we have \begin{equation}
        d^{(\eps/2)}_{\theta^\dagger}(s, g) \geq -\tilde{V}^*(s; g) - \frac{\eps}{2} \geq -V^*(s; g);
    \end{equation}
    \item for $s = g$, by \Cref{eq:lemma:ua}, we have \begin{equation}
        d^{(\eps/2)}_{\theta^\dagger}(s, g) = d^{(\eps/2)}_{\theta^\dagger}(s, s) = 0 = -V^*(s; s) = -V^*(s; g).
    \end{equation}
\end{itemize}

Hence, $d^{(\eps/2)}_{\theta^\dagger} \geq -V^*$ globally. Now it only remains to show that $d^{(\eps/2)}_{\theta^\dagger}$ satisfies the constraint.

For any transition $(s, a, s', r = R(s, s'))$ in $\mathcal{M}$, by \Cref{eq:lemma:ua}, \begin{align}
    d^{(\eps/2)}_{\theta^\dagger}(s, s') 
    & \leq -\tilde{V}^*(s;s') + \frac{\eps}{2} \\
    & \leq \frac{\eps}{2} -R(s, s')  + \frac{\eps}{2} \tag{since $s \rightarrow s'$ is also a valid path in $\tilde{\mathcal{M}}$ with cost $\frac{\eps}{2} -R(s, s')$}\\
    & = -r + \eps,
\end{align}
which means that $d^{(\eps/2)}_{\theta^\dagger}$ satisfies the constraint. 

Hence the desired $d^{(\eps/2)}_{\theta^\dagger}$ exists.

% $-(1 + \frac{\eps}{2}) V^*$ is a \qmet on $\mathcal{S}$. By universal approximation, there exists $d^{(\eps/2)}_{\theta^\dagger}$  such that \begin{equation}
%     \forall s, g,\quad \abs{ d^{(\eps/2)}_{\theta^\dagger}(s, g) + (1 + \frac{\eps}{2}) V^* } \leq \frac{\eps}{2}.
% \end{equation}
% Then, $\forall s, g$ such that $-V^*(s, g) \geq 1$, \begin{equation}
%     d^{(\eps/2)}_{\theta^\dagger}(s, g)
%     \geq -(1 + \frac{\eps}{2}) V^*(s; g) - \frac{\eps}{2} 
%     \geq -(1 + \frac{\eps}{2}) V^*(s; g) - \frac{\eps}{2}
%     \geq- V^*(s; g).
% \end{equation}

% And $\forall s, g$ such that $-V^*(s, g) < 1 \implies V^*(s, g) = 0$, \begin{equation}
%     d^{(\eps/2)}_{\theta^\dagger}(s, g) \geq 0 = - V^*(s; g),
% \end{equation}
% by \qmet properties.

% Hence, $d^{(\eps/2)}_{\theta^\dagger}$ satisfies the requirements.
    
\end{proof}

Now we are ready to prove \Cref{thm:qrl-func-approx,thm:qrl-func-approx-formal}.

\begin{proof}[Proof of \Cref{thm:qrl-func-approx,thm:qrl-func-approx-formal}]

Let $d^{(\eps/2)}_{\theta^*}$ be the solution to the relaxed problem. By the definition of the universal approximator, such solutions exist.  Moreover, we have \begin{equation}
    \forall s, g, \quad d^{(\eps/2)}_{\theta^*}(s, g) \leq -(1 + \eps) V^*(s; g),\label{eq:sol-ubd}
\end{equation}
by the constraint and triangle inequality.

Define \begin{equation}
    p \trieq \mathbb{P}[ d^{(\eps/2)}_{\theta^*}(s, g) < -(1 + \eps )V^*(s;g) - \delta ].\label{eq:sol-p-lbd}
\end{equation}

Then \begin{equation}
    \mathbb{E}[d^{(\eps/2)}_{\theta^*}(s, g)] \leq -(1+ \eps)\mathbb{E}[V^*(s; g)] - p \delta ,\label{eq:ineq1}
\end{equation} 
where we used \Cref{eq:sol-ubd,eq:sol-p-lbd}.

Let $d^{(\eps/2)}_{\theta^\dagger}$ be the \qmet from \Cref{lemma:relaxed-qrl-lb}. Then, by optimality, we must have \begin{equation}
    \mathbb{E}[d^{(\eps/2)}_{\theta^*}(s, g)] \geq  \mathbb{E}[d^{(\eps/2)}_{\theta^\dagger}(s, g)] \geq - \mathbb{E}[V^*(s; g)].\label{eq:ineq2}
\end{equation}

Combining \Cref{eq:ineq1,eq:ineq2}, we have \begin{equation}
    -(1+\eps)\mathbb{E}[V^*(s; g)] - p \delta \geq - \mathbb{E}[V^*(s; g)].
\end{equation}

Rearranging the terms, we have \begin{align}
    p  & \leq \frac{\eps}{\delta} \cdot (- \mathbb{E}[V^*] ).
    \label{eq:qrl-approx-final-lbd}
\end{align}

Combining \Cref{eq:qrl-approx-final-lbd,eq:sol-ubd} gives the desired result.
\end{proof}
\section{Experiment Details and Additional Results}\label{sec:expr-details}

All our results are aggregation from $5$ runs with different seeds.

We first discuss general design details that holds across all settings. For task-specific details, we discuss them in separate subsections below.

\myparagraph{QRL.} Across all experiments, we use $\eps = 0.25$, initialize Lagrange multiplier  $\lambda = 0.01$, and use Adam \citep{kingma2014adam}  to optimize all parameters. $\lambda$ is optimized via a softplus transform to ensure non-negativity. Our latent transition model $T$ is implemented in a residual manner, where \begin{equation}
    T(z, a) \trieq g_\phi(z, a) + z,
\end{equation}
and $g_\phi$ being a generic MLP with weights and biases of the last fully-connected layer initialized to all zeros. Unless otherwise noted, all networks are implemented as simple ReLU MLPs.  $p_\mathsf{state}$ is taken to be the beginning state of a random transition sampled from dataset / replay buffer. Unless otherwise noted, $p_\mathsf{goal}$ is taken to be the resulting state of a random transition sampled from dataset / replay buffer. For maximizing $d_\theta$, unless otherwise noted, we use the strictly monotonically increasing convex function \begin{equation}
    \phi(x) 
    \trieq -\mathrm{softplus}(500 - x, \beta=0.01) 
    = - 100 \times \mathrm{softplus}(5 - \frac{x}{100}).
\end{equation}

\myparagraph{MSG.} We follow the authors'suggestions, use $64$-critics, and tune the two regularizer hyperparameters over $\alpha \in \{0, 0.1, 0.5, 1\}$ and $\beta \in \{-4, -8\}$. For other hyperparameters, we use the same default values used in the original paper \citep{ghasemipour2022so}.

\subsection{Discretized \mntcar} 

\myparagraph{Discretization.} \mntcar state is parametrized by $\textsf{position} \in [-1.2, 0.6]$ and $\textsf{velocity} \in [-0.07, 0.07]$. For a dimension with values in interval $[l, u]$, we consider $160$ evenly spaced bins of length $(u - l) / 159$, with centers being \begin{equation}
    \left\{l + \frac{u-l}{159} \times k \colon k = 0, 1, 2, \dots, 159\right\}.
\end{equation}
After each reset and transition, we discretize each dimension of the state vector, so that future dynamics start from the discretized vector. To discretize a value, we find the bin it falls into, and replace it with the value of bin center. Note that the two bins at the two ends are centered at $u$ and $l$, respectively. So the two ends are exactly represented. Discretizing each dimension this way leads to $160\times160$ discrete states.

\myparagraph{Data.} In \mntcar, the original environment goal (top of hill) is a set of states with $\textsf{position} \in [0.5, 0.6]$ and $\textsf{velocity} \in [0, 0.07]$, where the agent is considered reaching that goal if it reaches any of those states. We adapt QRL and other goal-reaching methods to support this general goal following the procedure outlined in \Cref{sec:more-extensions}. Specifically, we augment the observation space to include an additional indicator dimension, which is $1$ only when representing this general goal. In summary, any original (discretized) state $s\trieq [u, v]$ becomes $\tilde{s}\trieq [u, v, 0]$, and $G \trieq [0.5, 0, 1]$ refers to this general goal. All critics and policies now takes in this augmented $3$-dimensional vector as input. For each encountered state $\tilde{s}$ that falls in this set, a new transition $(\tilde{s}, G)$ is added to the offline dataset.  The dataset includes $240$ such added transitions and $199{,}888$ transitions generated by running a random actor for $1{,}019$ episodes, where each episode terminate when the agent reaches top of hill or times out at $250$ timesteps. 

\myparagraph{Evaluation.}
For each target goal, we evaluate the planning performance starting from each of $160\times160$  states, with a budget of $200$ steps. At each step, the agent receives $-1$ reward until it reaches the goal. The episode return is then averaged over $160\times160$ states to compute the statistics. For the task of planning towards $9$ specific states, we say that agent reaches the goal if it reaches a $13 \times 13$ neighborhood centered around the goal state, and average the metrics over $9$ target goal states.
For QRL and Q-Learning, we did not train any policy network. Instead, the agents take the action that maximizes Q value (or minimizes distance) for simplicity. 

\myparagraph{Goal Distribution.}
For all multi-goal methods, wherever possible, we adopt a goal-sampling distribution as following: for $s_\mathsf{goal} \sim p_\mathsf{goal}$,
\begin{equation}
    s_\mathsf{goal} = \begin{cases}
    \textsf{\small resulting state from a random transition} & \text{ with probability $0.95$}\\
    [0.5, 0, 1] & \text{ with probability $0.05$}.
    \end{cases}\label{eq:dmc-pgoal}
\end{equation}

\myparagraph{QRL.} We use 3-1024-1024-1024-256 network for $f$ and (256+3)-1024-1024-1024-256 residual network for $T$, where $3$ represents the one-hot encoding of $3$ discrete actions. For $d_\theta$, we use a 256-1024-1024-1024-256 projector followed by an IQE-maxmean head with $16$ components, each of size $32$. $\mathcal{L}_\mathsf{transition}$ is optimized with a weight of $75$. Our learning rate is $0.3$ for $\lambda$ and $5\times 10^{-4}$ for the model parameters. We use a batch size of $4096$ to train $5\times10^5$ gradient steps. For all parameters except $\lambda$, we used cosine learning rate scheduling without restarting, decaying to $0$ at the end of training.

\myparagraph{Q-Learning.} We use $x$-1024-1024-1024-1024-1024-1024-3 networks for vanilla Q-Learning, where $x = 3$ in the single-goal setting, and $x=6$ in the multi-goal setting. The $3$ outputs represents estimated Q values for all $3$ actions. 

\myparagraph{Q-Learning with \Qmets. } We use the same encoder and projector architecture as QRL, as well as the same IQE specification. Additionally, to model the Q-function, we also add a 256-1024-1024-1024-(3$\times$256) transition model (which outputs the residual for each of the $3$ actions), and adopt QRL's transition loss with a weight of $5$. In other words, we replace the QRL's value learning objective with the Q-Learning temporal-difference objective (and keep the transition loss).  We use a discount factor of $0.95$, and update the target Q model every $2$ iterations with a exponential moving average factor of $0.005$. We use a learning rate of $0.001$ and a batch size of $4096$ to train $5\times10^5$ gradient steps.

\myparagraph{Contrastive RL.} We mostly follow the author's parameters for their offline experiments, using $x$-1024-1024-1024-$d_z$ encoders, where $x = (3 + 3)$ for the state-action encoder, $x = 3$ for the goal encoder, and $d_z$ is the latent dimension. We tune $d_z \in \{16, 64\}$ and choose $64$ for better performance. The policy training is modified to compute exactly the expected Q-value  (rather than using a reparametrized sample) from the policy's output action distribution, to accommodate the discrete action space. Since the dataset is generated from a random actor policy, we disable the behavior cloning loss. We train over $10^5$ gradient steps using a batch size of $1024$. We note that Contrastive RL requires a specific goal-sampling distribution, which we use instead of $p_\mathsf{goal}$ from \Cref{eq:dmc-pgoal}.

\myparagraph{Contrastive RL with \Qmets.}  We use the same encoder and projector architecture as QRL, as well as the same IQE specification. Similar to Q-Learning, we also add a residual transition model, which uses the same (256+3)-1024-1024-256  architecture as QRL's transition model, and adopt QRL's transition loss with a weight of $5$. In other words, we replace the QRL's value learning objective with the contrastive objective from Contrastive RL (and keep the transition loss). Contrastive RL objective estimates the on-policy Q-function with an extra goal-specific term determined by $p_\mathsf{goal}$ \citep{eysenbach2022contrastive}. Thus, we also learn a 256-1024-1024-1 model $c(z_g)$, where $z_g$ is the latent of goal $g$. Contrastive RL loss is computed with the sum of $c(z_g)$ and \qmet output. Other hyperparameters are identical to the vanilla Contrastive RL choices.

\myparagraph{MSG.} We follow the original paper and tune $\alpha \in \{0, 0.1, 0.5, 1\}$ and $\beta \in \{-4, -8\}$. After tuning, we select $\alpha=0.1$, $\beta=-4$ for both the single-goal and multi-goal setting. For relabelling, we find using random goals hurting performance. Hence, instead of $p_\mathsf{goal}$ from \Cref{eq:dmc-pgoal}, we use $[0.5, 0.5, 1]$ with probability $0.05$, and a future state from the same trajectory with probability $0.95$, where the future state is taken to be $\Delta t \geq 1$ steps away, where $\Delta t \sim \mathrm{Geometric}(0.3)$.

\myparagraph{Diffuser.} Diffuser's training horizon defines the length of trajectory segment used in training. Any trajectory with length shorter than this number won't be sampled at all for training. We tune the training horizon between $16$ (which includes almost all training trajectories) and $200$ (which excludes shorter trajectories from training but may better capture long-term dependencies), and choose $16$ due to its better performance in both evaluations.

\subsection{Offline \dFOURrl \maze}
\myparagraph{Evaluation.}
For each method, we evaluate both single-goal and multi-goal planning over $100$ episodes.

% iqe(2048,n=64)_maxmean/a0.01singleS+_relu_2critics_enc=3x1024_prj=2x1024_zED256_bsz4096dlt1dtol0.25MSE_inf,onx1(12),aDRo,S+(offset=500)_dyn=d10gd1gx1gy1,3x1024_pi=rd1+bc0.05g0.99,4x1024_ep2e5steps_lr0.0005_pilr3e-05_alr0.01adam_s1438

\myparagraph{QRL.} We use 4-1024-1024-1024-256 network for $f$ and (256+2)-1024-1024-1024-256 residual network for $T$, where $2$ is the action dimension. For $d_\theta$, we use a 256-1024-1024-2048 projector followed by an IQE-maxmean head with $64$ components, each of size $32$. $\mathcal{L}_\mathsf{transition}$ is optimized with a weight of $1$. Our learning rate is $0.01$ for $\lambda$,  $5\times 10^{-4}$ for the critic parameters, and $3\times 10^{-5}$ for the policy parameters. We use a batch size of $4096$ to train $2\times10^5$ gradient steps. Inspired by Contrastive RL \citep{eysenbach2022contrastive}, we augment policy training with an additional behavior cloning loss of weight $0.05$ (towards a goal that is $\Delta t \geq 1$ steps in the future from the same trajectory, for $\Delta t \sim \mathrm{Geometric}(0.99)$).

\myparagraph{Contrastive RL.} We mostly follow the author's parameters for their offline experiments, using $x$-1024-1024-1024-16 encoders, where $x = (4 + 2)$ for the state-action encoder, and $x = 4$ for the goal encoder, and $d_z$ is the latent dimension, as well as a behavior cloning loss of weight $0.05$. We train over $1.5 \times 10^5$ gradient steps using a batch size of $1024$. We note that Contrastive RL requires a specific goal-sampling distribution, which we use instead of $p_\mathsf{goal}$ from \Cref{eq:dmc-pgoal}.

\myparagraph{MSG.} For single-goal results, we report the evaluations from the original paper. For multi-goal tasks, we use the same architectures with relabelling, and  tune $\alpha \in \{0, 0.1, 0.5, 1\}$ and $\beta \in \{-4, -8\}$, following the procedure from original paper. After tuning, we use $\alpha=0.1$ and $\beta=-8$ for the \texttt{large} maze, $\alpha=0.5$ and $\beta=-4$ for the \texttt{medium} maze, and $\alpha=0.1$ and $\beta=-8$ for the \texttt{umaze} maze. For relabelling, we sample goal state a future state from the same trajectory, where the future state is taken to be $\Delta t \geq 1$ steps away, where $\Delta t \sim \mathrm{Geometric}(0.3)$.

\myparagraph{MPPI with QRL Value.} We run MPPI in the QRL's learned dynamics and value function with a planning horizon of $5$ steps, $10{,}000$ samples per step, and the QRL Q-function  (via the QRL dynamics and value function) as reward in each step. The noise variance to sample and explore actions is $\sigma^2=1$. We experimented $\lambda \in \{0.1, 0.01\}$, a regularizer penalizing the cost of control noise, and use $\lambda = 0.01$ due to its slightly superior performance.

\myparagraph{Diffuser.} We strictly follow the original paper's parameters for \maze experiments. For planning with sampled actions, each Diffuser sample yields many actions, so we replan after using up all previously sampled actions (similar to open-loop planning). In our experience, replanning at every timestep is extremely computationally costly without observed improvements. For QRL value planning, we guide Diffuser sampling for minimizing the learned \qmet distance towards goal state (in addition to its existing goal-conditioning) with a weight of $0.1$ over $4$ guidance steps at each sampling iteration. Since each Diffuser sample is a long-horizon trajectories refined over many iterations, guiding at each timestep of the trajectory is computationally expensive. Therefore, we gather state-action pairs from every $5$ timesteps as well as the last step of the trajectory, and feed these pairs into learned QRL value function to compute the average QRL values as guidance.

\subsection{Online GCRL}\label{sec:app:gcrl-details}

\begin{figure*}[t]
    \centering
    % \vspace{-3pt}
    % \includegraphics[scale=0.417, trim=5 0 0 0, clip]{figures/online_more_final_arXiv.pdf}%
    \includegraphics[scale=0.417, trim=5 0 0 0, clip]{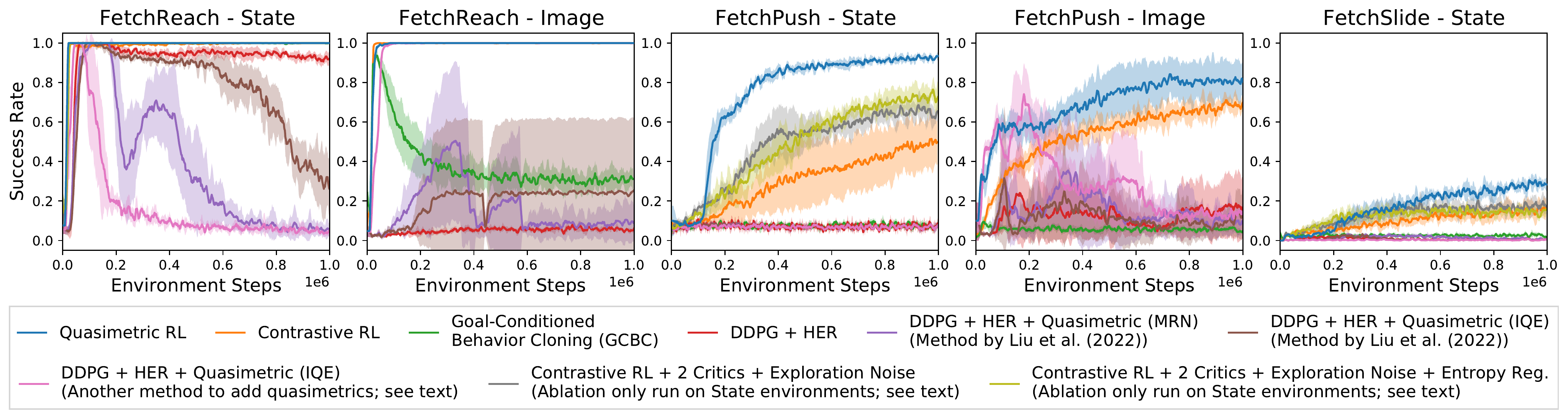}%
    \vspace{-8pt}
    \caption{Online learning performance on GCRL benchmarks, including an alternative method to integrate \qmets in DDPG and a variant of Contrastive RL trained with two critics and exploration action noise on state-based settings. No method has access to ground truth reward function. QRL  still consistently outperforms the baseline methods, learning both faster and better. \texttt{FetchSlide} with image observation is not shown because no method reaches a non-trivial success rate. See \Cref{sec:app:gcrl-details} for details of the additional baselines.
    }
    \label{fig:online-more}%
    \vspace{-8pt}
\end{figure*}

\myparagraph{Environment.} For \texttt{FetchReach} and \texttt{FetchPush}, we strictly follow Contrastive RL experimental setups \citep{eysenbach2022contrastive} to generate initial and goal states/images. The image observations are RGB with $64 \times 64$ resolution. For \texttt{FetchSlide}, we adopt a similar strategy and generate goal states where object position dimensions are set to the target location and other dimensions are set to zeros. We are unable to get any method to reach a non-trivial success rate on \texttt{FetchSlide} with image observation despite tuning hyperparamteres and image rendering. We thus omit this setting in results.

\myparagraph{Evaluation.} We evaluate each method for $50$ episodes every $2000$ environment steps (\ie, $40$ episodes). Following standard practice, we mark an episode as successful if the agent completes the task at any timestep within the time limit ($50$ steps). For clearer visualizations in \Cref{fig:online,fig:online-more}, the success rates curves are smoothed with a sliding window of length $5$ before gathering across $5$ seeds, similar to visualizations in \citep{liu2022metric}. For comparing sample efficiencies between Contrastive RL and QRL, we look at the smoothed success rates from both methods, find the sample size where QRL first exceeds Contrastive RL's final performance at $10^6$ samples, and compute the sample size ratio. 

\myparagraph{Processing Image Observations.} To process image inputs, all compared methods use the same backbone convolutional architecture from \citep{mnih2013playing} to encode the input image into a $1024$-dimensional flat vector. We adopt this approach from Contrastive RL \citep{eysenbach2022contrastive}. For different modules in a method, each module uses an independent copy of this backbone (of same architecture but different set of parameters). For modules that takes in two observations (\eg, policy network in all methods and monolithic Q-functions in vanilla DDPG), the same backbone processes each input into a flat vector, and the concatenated $2048$-dimensional vector is fed into later parts of the module (which is usually an MLP). Other modules only take in a single observation and simply maps the processed $1024$-dimensional vector to the output in a fashion similar to the fully-connected head of convolutional nets (\ie, passing through an activation function and then an MLP). In architecture descriptions below, we omit this backbone part for simplicity, and use $x$ to denote the state dimension for state-based observations and backbone output dimension (\ie, $1024$) for image-based observations .

% iqe(2048,n=64)_maxmean/Expl0.3_a0.01_relu_2critics_enc=2x512_prj=2x2048_zD128_bsz256dlt1dtol
% MSE_inf,onx1(12),aDRo,S+(offset=15,beta=0.1)_dyn=d10gd1gx1gy1,2x512_pi=rd0.5futw0.5g0.99+ent,3x512_steps1000000_lr0.0001_pilr3e-05_alr0.01adam_s

    % steps: 1000000
    % log_steps: 100
    % eval_steps: 2000
    % save_steps: 50000
    % intervals_use_optim_steps: false
    % num_prefill_episodes: 200
    % num_samples_per_cycle: 500
    % num_rollouts_per_cycle: 10
    % num_eval_episodes: 50
\myparagraph{QRL (State-based Observations).} 
We use a $x$-512-512-128 network  for $f$ and a (128+4)-512-512-128 residual network for $T$, where  $4$ is the action dimension. For $d_\theta$, we use a 128-512-2048 projector followed by an IQE-maxmean head with $64$ components, each of size $32$. We use $x$-512-512-8 network for policy, where $x$ is the input size and $8$ parametrizes a tanh-transformed diagonal Normal distribution. $\mathcal{L}_\mathsf{transition}$ is optimized with a weight of $0.1$. Our learning rates are $0.01$ for $\lambda$, $1\times 10^{-4}$ for the model parameters, and $3\times 10^{-5}$ for the policy parameters. We use a batch size of $256$ in training. We prefill the replay buffer with $200$ episodes from a random actor, and then iteratively perform (1) generating $10$ rollouts  and (2) optimizing QRL objective for $500$ gradients steps.  We use $\mathcal{N}(0, 0.3^2)$-perturbed action noise in exploration. For the adaptive entropy regularizer \citep{haarnoja2018soft}, we regularize policy to have target entropy $-\mathrm{dim}(\mathcal{A})$, where the entropy regularizer weight is initialized to be $1$ and optimized in log-space with a learning rate of $3\times 10^{-4}$.  Since the environment has much shorter horizon (each episodes ends at $50$ timesteps), we instead use a different affine-transformed softplus for maximizing $d_\theta$, where $\phi(x) 
    \trieq -\mathrm{softplus}(15 - x, \beta=0.1)$.

\myparagraph{QRL (Image-based Observations).}  All settings are the same as QRL for state-based observations \emph{except}  a few changes: \begin{itemize}[itemsep=-2.5pt,topsep=-4.5pt]
    \item We use the convolutional backbone followed by a $x$-512-128 network for encoder $f$.
    \item We optimize $\mathcal{L}_\mathsf{transition}$ with an increased weight of $10$ (since the dynamics aren't fully deterministic). 
    \item We update the models less frequently with $125$ gradient steps every $10$ rollouts. Contrastive RL uses the same reduced update frequency for image-based observations \citep{eysenbach2022contrastive}, which we observe also has benefits for QRL\footnote{This is potentially related to the lost of capacity phenomenon observed generally in  RL algorithms \citep{doro2022sample}}.
\end{itemize}
\vspace{3pt}

\myparagraph{Contrastive RL.} We strictly follow the original paper's experiment settings \citep{eysenbach2022contrastive}, which does not use two critics or action noise for exploration, and only uses entropy regularizer for image-based observations. For a comparison, we also run Contrastive RL with these techniques added on the state-based environments. As shown in \Cref{fig:online-more}, while they do sometimes improve performance, they do not completely explain the gap between QRL and Contrastive RL. Hence, the improvement of QRL over Contrastive RL indeed (partly) comes from fundamental algorithmic differences. Since Contrastive RL estimates on-policy values, it could be more sensitive adding exploration noises, which degrades the dataset. QRL, however, is conceptually exempt from this issue, since it estimates optimal values.

\myparagraph{Goal-Conditioned Behavior Cloning (GCBC).} We strictly follow the hyperparameter setups for the GCBC baseline in the Contrastive RL paper \citep{eysenbach2022contrastive}.

\myparagraph{DDPG + HER.}
We mostly follow the experiment setup in the MRN paper \citep{liu2022metric}. However, we do not give HER access to reward functions for fair comparison. Instead, HER relabels transition rewards based on whether the state equals the target goal state, which is exactly the same reward structure other method uses (QRL, Contrastive RL and GCBC). 

\myparagraph{DDPG + HER + \Qmets (Method by \citet{liu2022metric}).} We strictly follow the MRN paper \citep{liu2022metric} to modify DDPG to include \qmets, which is slightly different from our modifications to Q-Learning on offline \mntcar, but was also shown to be empirically beneficial in online learning \citep{liu2022metric}. We follow \citet{liu2022metric} for MRN hyperparameters, and use the same IQE hyperparameters as QRL. 

\myparagraph{DDPG + HER + \Qmets (Another method to add \qmets).} We show \emph{additional} results comparing QRL to a different approach to integrate \qmets into DDPG. This approach is different from the one by \citet{liu2022metric} but similar to our modifications to Q-Learning on offline \mntcar that attain good performance in that task.  We adapt the architecture choices by \citet{liu2022metric} and QRL. Specifically, we use a $x$-512-512-128 network for encoder $f$ and  (128+4)-512-512-128 residual network for $T$. For $d_\theta$, we use a 128-2048-2048-2048 projector followed by an IQE-maxmean head with $64$ components, each of size $32$.  We adopt QRL's transition loss with a weight of $5$. In other words, we replace the QRL's value learning objective with the DDPG temporal-difference objective (and keep the transition loss). All other hyperparameters follow the same choices in method by \citet{liu2022metric}. This approach performs extremely poorly on this more challenging set of environments, suggesting that it is unable to scale to more complex continuous-control settings.

As shown in \Cref{fig:online-more}, QRL greatly outperforms both approaches to integrate DDPG and \qmets, showing consistent advantage of the QRL objective over Q-Learning's temporal-difference objective.

\end{document}